%% file: main.tex
\documentclass{article}
\usepackage[preprint]{neurips_2021}

\input{content/settings}

\def\authornotes{1pt}

\ifdim\authornotes=1pt
\newcommand{\ynote}[1]{\footnote{\color{ForestGreen}Yeshwanth: #1}}
\else
\newcommand{\ynote}[1]{}
\fi

\usepackage{tikz}
\allowdisplaybreaks

\begin{document}

\title{Adversarial Examples in Multi-Layer Random ReLU Networks}

\author{%
  Peter L. Bartlett \\
  Department of Electrical Engineering and Computer Science\\
  Department of Statistics\\
  UC Berkeley \\
  \And 
  S\'ebastien Bubeck \\
  Microsoft Research Redmond \\
  \And 
  Yeshwanth Cherapanamjeri \\
  Department of Electrical Engineering and Computer Science \\
  UC Berkeley
}



\maketitle

\begin{abstract}
We consider the phenomenon of adversarial examples in ReLU networks
with independent gaussian parameters.  For networks of constant depth
and with a large range of widths (for instance, it suffices if the
width of each layer is polynomial in that of any other layer), small
perturbations of input vectors lead to large changes of outputs.  This
generalizes results of Daniely and Schacham (2020) for networks of
rapidly decreasing width and of Bubeck et al (2021) for two-layer
networks. The proof shows that adversarial examples arise in these
networks because the functions that they compute are very close to
linear. Bottleneck layers in the network play a key role: the minimal
width up to some point in the network determines scales and
sensitivities of mappings computed up to that point.  The main result
is for networks with constant depth, but we also show that some
constraint on depth is necessary for a result of this kind, because
there are suitably deep networks that, with constant probability,
compute a function that is close to constant.
\end{abstract}

\input{content/intro}
\input{content/relu_proof_corrected}
\input{content/lower}

\bibliographystyle{alpha}
\bibliography{main}

\appendix
\input{content/appendix}

\end{document}

%% file: content/settings.tex
\usepackage{amsmath,amsthm,amssymb}
\usepackage{fullpage}
\usepackage{mathtools,bbm,xspace}
\usepackage{mathrsfs}
\usepackage{bm}

\usepackage[utf8]{inputenc} 
\usepackage[T1]{fontenc}    
\usepackage{url}            
\usepackage{booktabs}       
\usepackage{amsfonts}       
\usepackage{nicefrac}       
\usepackage{microtype}      
\usepackage{algorithm}
\usepackage{algpseudocode}
\usepackage[dvipsnames]{xcolor}
\usepackage{subfig}
\usepackage{ifxetex}

\usepackage{verbatim}

\usepackage[pagebackref]{hyperref}
\hypersetup{
pagebackref,
letterpaper=true,
colorlinks=true,
urlcolor=blue,
linkcolor=blue,
citecolor=Black,
}
\usepackage[capitalise,nameinlink]{cleveref}
\crefname{lemma}{Lemma}{Lemmas}
\crefname{fact}{Fact}{Facts}
\crefname{theorem}{Theorem}{Theorems}
\crefname{corollary}{Corollary}{Corollaries}
\crefname{claim}{Claim}{Claims}
\crefname{example}{Example}{Examples}
\crefname{problem}{Problem}{Problems}
\crefname{definition}{Definition}{Definitions}
\crefname{assumption}{Assumption}{Assumptions}
\crefname{subsection}{Subsection}{Subsections}
\crefname{section}{Section}{Sections}

\newtheorem{theorem}{Theorem}[section]

\newtheorem{lemma}[theorem]{Lemma}
\newtheorem{corollary}[theorem]{Corollary}
\newtheorem{fact}[theorem]{Fact}

\theoremstyle{definition}

\newtheorem{exercise-easy}[theorem]{Exercise}
\newtheorem{exercise-med}[theorem]{Exercise}
\newtheorem{exercise-hard}[theorem]{Exercise$^\star$}
\newtheorem{claim}[theorem]{Claim}

\DeclareMathOperator*{\argmin}{arg\,min}

\DeclareMathOperator*{\Tr}{\bm{Tr}}
\DeclareMathOperator*{\sgn}{sgn}

\DeclarePairedDelimiterX{\norm}[1]{\lVert}{\rVert}{#1}

\newcommand{\proj}{\mc{P}}

\DeclarePairedDelimiterX{\abs}[1]{\lvert}{\rvert}{#1}
\DeclarePairedDelimiterX{\inp}[2]{\langle}{\rangle}{#1, #2}

\DeclarePairedDelimiterX{\infdivx}[2]{(}{)}{%
  #1\;\delimsize\|\;#2%
}
\DeclareMathOperator*{\sign}{sign}

\newcommand{\mc}[1]{\mathcal{#1}}
\newcommand{\mb}[1]{\mathbb{#1}}
\newcommand{\mrm}[1]{\mathrm{#1}}
\newcommand{\wt}[1]{\widetilde{#1}}

\newcommand{\lprp}[1]{\left(#1\right)}
\newcommand{\lbrb}[1]{\left\{#1\right\}}
\newcommand{\lsrs}[1]{\left[#1\right]}

\newcommand{\E}{\mb{E}}
\renewcommand{\P}{\mb{P}}
\newcommand{\R}{\mb{R}}
\newcommand{\N}{\mb{N}}
\renewcommand{\S}{\mb{S}}
\newcommand{\B}{\mb{B}}

\newcommand{\eps}{\varepsilon}

\newcommand{\poly}{\mathrm{poly}}

\newcommand{\dmax}{d_{\mrm{max}}}
\newcommand{\dmin}{d_{\mrm{min}}}
\newcommand{\dmini}{d^i_{\mrm{min}}}

\newcommand{\epsnetres}{1 / (\dmax^\ell)^{32}}
\newcommand{\gddiff}{\mrm{Diff}}

%% file: content/intro.tex
\section{Introduction and Main Result}
\label{sec:intro}

Since the phenomenon of adversarial examples was first
observed in deep networks~\cite{szegedy14}, there
has been considerable interest in why this extreme
sensitivity to small input perturbations arises in deep
networks~\cite{adversarialExamples,shamir19,bubeck19,daniely20,bubeck21}
and how it can be detected and
avoided~\cite{carlini17,advDetected,feinman17,madry18,qin19}.
Building on the work of Shamir et al~\cite{shamir19},
Daniely and Schacham~\cite{daniely20} prove
that small perturbations (measured in the euclidean norm) can
be found for any fixed input and most gaussian
parameters in certain ReLU networks---those in which
each layer has vanishing width relative to the
previous layer---and conjectured the same result without
this strong constraint on the architecture.  Bubeck,
Cherapanamjeri, Gidel and Tachet des Combes~\cite{bubeck21} prove
that the same phenomenon occurs in general two-layer ReLU networks, and
give experimental evidence of its presence in deeper ReLU networks.

In this paper, we prove that adversarial examples also arise in
deep ReLU networks with random weights for a wide variety of network
architectures---those with constant depth and polynomially-related
widths.  The key fact underlying this phenomenon was already
observed in~\cite{daniely20}: a high-dimensional linear function
$f(x)=w^\top x$ with input $x\not=0$ and random parameter vector
$w$ with a uniformly chosen direction will satisfy $\|\nabla f(x)\|
\,\|x\|\gg |f(x)|$ with high probability. This implies the existence of
a nearby adversarial example for this linear function: a perturbation
of $x$ of size $|f(x)|/\|\nabla f(x)\|\ll\|x\|$ in the direction
$-f(x)\nabla f(x)$ will flip the sign of $f(x)$.  This observation
can be extended to nonlinear functions that are locally almost
linear. Indeed, it is easy to show that for all $x,u\in\mb{R}^d$,
      \[
        \left|f(x+u)-(f(x) + \langle u,\nabla f(x)\rangle)\right|
        \le \|u\|\sup\left\{\left\|
          \nabla f(x)-\nabla f(x+v)\right\|:
          v\in\mb{R}^d,\,\|v\|\le\|u\|\right\},
      \]
and thus to demonstrate the existence of an adversarial example near
$x$ for a function $f$, it suffices to show the smoothness property:
  \begin{equation}\label{eq:smoothness}
    \text{for all } v\in\mb{R}^d \text{ with } \|v\|\lesssim
      |f(x)|/\|\nabla f(x)\|,\qquad
      \left\| \nabla f(x)-\nabla f(x+v)\right\| \ll \|\nabla f(x)\|.
  \end{equation}
We show that for a deep ReLU network with random parameters and a
high-dimensional input vector $x$, there is a relatively large ball
around $x$ where the function computed by the network is very likely
to satisfy this smoothness property.  Thus, adversarial examples
arise in deep ReLU networks with random weights because the functions
that they compute are very close to linear in this sense.

It is important to notice that ReLU networks are not smooth in a
classical sense, because the nondifferentiability of the ReLU
nonlinearity implies that the gradient can change abruptly. But for
the smoothness condition~\eqref{eq:smoothness}, it suffices to have
$\left\| \nabla f(x)-\nabla f(x+v)\right\|\le \epsilon + \phi(\|v\|)$,
for some increasing function $\phi:\mb{R}_+\to\mb{R}_+$,
provided that $\epsilon+\phi(\|v\|)\ll\|\nabla f(x)\|$.
We prove an inequality like this for ReLU networks, where the
$\epsilon$ term decreases with width.

Consider a network with input dimension $d$, $\ell+1$ layers,
a single real output, and complete connections between
layers.  Let $d_1, \dots, d_\ell$ denote the dimensions of
the layers. The network has independent random weight matrices
$W_i\in\mb{R}^{d_i\times d_{i-1}}$ for $i\in[\ell+1]$, where we set
$d_0=d$ and $d_{\ell+1}=1$. For input $x\in\mb{R}^d$, the network
output is defined as follows:
\begin{gather}
    f(x) = W_{\ell + 1}\cdot \sigma (W_{\ell}\cdot \sigma
        (W_{\ell - 1}\cdot \sigma ( \cdots \sigma (W_1\cdot x) \cdots
        ))) \text{ where } \sigma (x)_i = \max \{x_i, 0\} \notag \\
    W_{\ell + 1} \sim \mc{N} (0, I / d_\ell) \text{ and }
        \forall i\in[\ell],\,
        (W_i)_{j,k} \overset{i.i.d}{\thicksim} \mc{N} (0, 1 /
        d_{i-1})
    \label{eq:nn_def} \tag{NN-DEF}.
\end{gather}
Note that the scale of the parameters is chosen so that all the
real-valued signals that appear throughout the network have roughly
the same scale. This is only for convenience: because the ReLU is
positively homogeneous (that is, for $\alpha>0$, $\sigma(\alpha x)
= \alpha\sigma(x)$), the scaling is not important for our results;
the $1/d_{i-1}$ in~\eqref{eq:nn_def} could be replaced by any
constant without affecting the ratio between the norm of an input
vector and that of a perturbation required to change the sign of
the corresponding output.

The following theorem is the main result of the paper.
\begin{theorem}\label{thm:main}
  Fix $\ell \in \N$. There are constants $c_1,c_2,c_3$ that depend on
  $\ell$ for which the following holds.
  Fix $\delta \in (0, 1)$ and let $f(\cdot)$ be an $(\ell + 1)$-layer ReLU
  neural network defined by~\eqref{eq:nn_def} with input dimension $d$
  and intermediate layers of width $\{d_i\}_{i = 1}^\ell$. Suppose
  that the widths satisfy
  \begin{gather*}
    \dmin \geq c_1 (\log \dmax)^{c_2} \log 1 / \delta \text{ where }
    \dmin = \min \lbrb{\{d_i\}_{i = 1}^\ell, d}, \dmax = \max
    \lbrb{\{d_i\}_{i = 1}^\ell, d}.
  \end{gather*}
Then for any fixed input $x \neq 0$, with probability at least $1 -
\delta$,
  \begin{equation*}
    \abs{f(x + \eta \nabla f(x))} \geq \abs{f(x)} \text{ and } \sign
    (f(x + \eta \nabla f(x))) \neq f(x),
  \end{equation*}
for an $\eta$ satisfying
  \[
    \frac{\norm{\eta \nabla f(x)}}{\norm{x}}
      \leq c_3 \sqrt{\frac{\log 1 / \delta}{d}}. 
  \]
It suffices to choose $c_1=(C_1 \ell)^{c_2}$, 
$c_2=C_2\ell$, $c_3=C_3^\ell$, for some absolute constants
$C_1$, $C_2$, $C_3$.
\end{theorem}
This theorem concerns networks of fixed depth, and the constants in
the size of the perturbation and in the requirement on the network
width are larger for deeper networks. We also prove a converse
result that illustrates the need for some constraint on the
depth. Theorem~\ref{thm:lower-in-body} shows that when the depth
is allowed to grow polynomially in the input dimension $d$, the
function computed by a random ReLU network is essentially constant,
which rules out the possibility of adversarial examples.

The heart of the proof of Theorem~\ref{thm:main} is to show a
smoothness property like~\eqref{eq:smoothness}.  It exploits
a decomposition of the change of gradient between two input
vectors.  Define $H_i:\mb{R}^d\to\mb{R}^{d_i\times d_i}$ as
$H_i(x)_{jk}=\bm{1}\{j=k,\,v_i(x)_j\ge 0\}$ with $v_i(x)=W_i \sigma
(\cdots \sigma (W_1 x))$.  For two input vectors $x,y\in\mb{R}^d$,
we will see in Section~\ref{ssec:local_grad_smooth} that we can
decompose the change of gradient
as
\begin{align*}
    \nabla f(x) - \nabla f(y)
    &= \sum_{j = 1}^\ell W_{\ell + 1}
      \lprp{\prod_{i = \ell}^{j + 1} H_i(x) W_i}\cdot
      (H_j(x) - H_j(y)) W_j \cdot
      \lprp{\prod_{i = j - 1}^1 H_i(y) W_i}.
\end{align*}
(Here and elsewhere, indices of products of matrices run backwards, so
$\prod_{i=j}^k M_i=I$ when $j<k$.)
For the $j$th term in the decomposition, we need to control the
scale of: the gradient of the mapping from the
input to the output of layer $j$, the change in the layer $j$
nonlinearity $H_j(x)-H_j(y)$, and the gradient from layer $j$ to
the output.  It turns out that controlling these quantities depends
crucially on the width of the narrowest layer before layer $j$---we
call this the {\em bottleneck layer} for layer $j$.  This width
determines the dimension of the image at layer $j$ of a ball in the
input space.  In proving bounds on gradients and function values
that hold uniformly over pairs of nearby vectors $x$ and $y$, this
dimension---the width of the bottleneck layer---dictates the size of
a discretization (an $\epsilon$-net) that is a crucial ingredient
in the proof of these uniform properties.  Our analysis involves
working separately with the segments between these bottleneck layers.
We show that for an input $x\in\mb{R}^d$ satisfying $\|x\|=\sqrt d$
and any $y$ in a ball around $x$, with high probability $\|\nabla
f(x) - \nabla f(y)\|=o(1)$, but $|f(x)|$ is no more than a constant
and $\|\nabla f(x)\|$ is at least a constant. This implies the
existence of a small ($o(\|x\|)$) perturbation of $x$ in the
direction $-f(x)\nabla f(x)$ that flips the sign of $f(x)$.

These results suggest several interesting directions for future work.
First, our results show that for high-dimensional inputs, adversarial
examples are inevitable in random ReLU networks with constant depth,
and unlikely in networks with polynomial depth. Beyond this, we
do not know how the sensitivity to input perturbations decreases
with depth.  Similarly, both results are restricted to networks
with subexponential width, and it is not clear what happens for very
wide networks. Finally, we show that networks with random weights
suffer from adversarial examples because their behavior is very
similar to that of random linear functions. It would be worthwhile
to determine whether randomly initialized trained networks retain
this nearly linear behavior, and hence suffer from adversarial
examples for the same reason.

%% file: content/relu_proof_corrected.tex
\section{Proof of Main Theorem}
\label{sec:proof_relu}

In this section, we provide an outline of the proof of
\cref{thm:main}. As described above, we will prove
our result first by showing that the gradient at $x$ has large
norm and changes negligibly in a large ball around $x$. The first
condition is established in \cref{ssec:f_grad_anticonc}.  The second
step is more intricate.  First, we prove the decomposition of
the gradient differences in \cref{ssec:grad_decomp}. Then, in
\cref{ssec:scale_pres_local_nbrhood}, we track the scale of the
ball around $x$ as it propagates through the network.  Finally,
in \cref{ssec:local_grad_smooth}, we use this result to bound the
terms in the decomposition of the gradient differences to show that
our network is locally linear. For the rest of the proof, unless
otherwise stated, we consider a fixed $x\in\mb{R}^d$, and we assume:
\begin{equation*}
    \dmin \geq (C \ell \log \dmax)^{240\ell} \log 1 / \delta, \quad
    \norm{x} = \sqrt{d} \text{ and } R \coloneqq
    \frac{\sqrt{\dmin}}{(\ell \log \dmax)^{80\ell}} = \Omega
    \left((\ell \log \dmax)^{40\ell}\right).
\end{equation*}
Additionally, we randomize the activations of neurons whenever they
receive an input of $0$. This does not change the behavior of the neural network in terms of its output or the images of the input through the layers of the network but greatly simplifies our proof. For $x \in \mb{R}^d$, we let:
\begin{equation*}
    x_0 \coloneqq x, 
    \ \wt{f}_i (x) \coloneqq W_i f_{i - 1} (x),
    \ (D_{i} (y))_{j,k} = 
    \begin{cases}
        1 & \text{w.p } \frac{1}{2} \text{ if } j = k,\ y_j = 0, \\
        1 & \text{ if } j = k,\ y_j > 0, \\
        0 & \text{otherwise,} 
    \end{cases}
    \ f_i(x) = D_i (\wt{f}_i (x)) \wt{f}_i (x). 
\end{equation*}
Our first key observation is that the randomization in the activation
units allows us the following distributional equivalences,
proved in \cref{ssec:dis_equiv_proof}. 
\begin{lemma}
    \label{lem:dist_equiv}
    Let $m \in \N$, $\{d_i\}_{i = 0}^{m} \subset \R^d$ and $W_i \in \R^{d_i \times d_{i - 1}}$ be distributed such that each entry of $W_i$ is drawn iid from any symmetric distribution. Then, defining for $x \in \R^{d_0}$:
    \begin{gather*}
        \begin{aligned}
            h_0(x) &= x,\\
            \wt{h}_i (x) &= W_i h_{i - 1} (x) \\
            h_i (x) &= D_i (\wt{h}_i (x)) \wt{h}_i (x)
        \end{aligned}
        \text{ where } (D_i(y))_{i,j} = 
        \begin{cases}
            1, &\text{if } j = k \text{ and } y_j > 0\\
            1, &\text{with probability } \frac{1}{2} \text{ if } j = k,\ y_j = 0 \\
            0, &\text{otherwise}
        \end{cases}
    \end{gather*}
    we have the distributional equivalences for any $x \neq 0$ and
    fixed diagonal matrices $B_1,\ldots,B_m$:
    \begin{gather*}
        W_m \prod_{j = m - 1}^1  (D_j (\wt{h}_j (x)) + B_j) W_j \overset{d}{=} W_m \prod_{j = m - 1}^1  (D_j + B_j) W_j \\
        \norm*{\prod_{j = m}^1  (D_j (\wt{h}_j (x)) + B_j) W_j} \overset{d}{=} \norm*{\prod_{j = m}^1  (D_j + B_j) W_j} \\
        \text{ where } (D_j)_{k,l} = 
        \begin{cases}
            1, &\text{with probability } 1/2 \text{ if } k = l \\
            0, &\text{otherwise}
        \end{cases}
    \end{gather*}
\end{lemma}

\subsection{Concentration of Function Value and Gradient at a Fixed Point}
\label{ssec:f_grad_anticonc}

We first present a simple lemma that shows that the gradient at $x$ is
at least a constant and that its output value is bounded. The proof
gives an illustration of how \cref{lem:dist_equiv} will be used through
the more involved proofs in the paper.

\begin{lemma}
    \label{lem:f_grad_conc}
    For some universal constant $c$,
    with probability at least $1 - \delta$ we have:
    \begin{equation*}
        \left|f(x)\right| \leq c2^{\ell} \sqrt{\log 1 / \delta}
        \text{ and } \norm{\nabla f(x)} \geq \frac{1}{2^{\ell +
        1}}.
    \end{equation*}
\end{lemma}
\begin{proof}
    Note that
    \begin{equation*}
        \nabla f (x) = W_{\ell + 1} \prod_{i = \ell}^1 D_i (\wt{f}_i (x)) W_i \text{ and } f(x) = \nabla f (x) x.
    \end{equation*}
    And we have from \cref{lem:dist_equiv},
    $\nabla f(x) \overset{d}{=} \tilde{W}_{\ell + 1}
    \prod_{i = \ell}^1 D_i \tilde{W}_i$, where
    \begin{equation*}
        (D_i)_{j,k} = \begin{cases}
                        1 & \text{with probability $1/2$ if $j=k$,} \\
                        0 & \text{otherwise,}
                      \end{cases}
        \qquad\{W_i\}_{i = 1}^{\ell + 1} \overset{d}{=}
        \{\tilde{W}_i\}_{i = 1}^{\ell + 1}.
    \end{equation*}
    Therefore, it suffices to analyze the random vector
    $\tilde{W}_{\ell + 1} \prod_{i = \ell}^1 D_i \tilde{W}_i$. We
    first condition on a favorable event for the $D_i$. Note that we
    have by the union bound and an application of Hoeffding's
    inequality (e.g.,~\cite[Theorem 2.8]{blm})
    that:
    \begin{equation*}
        \forall i \in [\ell]: \Tr D_i \geq \frac{d_i}{3} \text{ with probability at least } 1 - \delta / 4,
    \end{equation*}
    since $d_i\ge c\log(4\ell/\delta)$. We now condition on the $D_i$ and note that:
    \begin{gather*}
        \norm*{\wt{W}_{\ell + 1} \prod_{i = \ell}^1 D_i \wt{W}_i} \overset{d}{=} \norm*{W^\dagger_{\ell + 1} \prod_{i = \ell}^1 W^\dagger_{i}} \text{ where } \\
        \forall i \in \{2, \dots, \ell\}: W^\dagger_i \in \R^{\Tr D_i \times \Tr D_{i - 1}},\ W^\dagger_{\ell + 1} \in \R^{1 \times \Tr D_\ell},\ W^\dagger_{1} \in \R^{\Tr D_1 \times d} \\
        (W^\dagger_1)_{j, :} \thicksim \mc{N} (0, I / d),\
        W^\dagger_{\ell + 1} \thicksim \mc{N} (0, I / d_\ell),\
        \forall i \in \{2, \dots, \ell + 1\}: (W^\dagger_i)_{j, :}
        \thicksim \mc{N} \lprp{0, I/d_{i - 1}}.
    \end{gather*}
    From the above display, we obtain from \cref{thm:tsirelson}
    and its corollary~\ref{cor:tsirelson},
    \begin{align*}
        \norm*{W^\dagger_{\ell + 1} \prod_{i = \ell}^1 W^\dagger_i} &\geq \frac{1}{2} \cdot \norm*{W^\dagger_{\ell + 1} \prod_{i = \ell}^2 W^\dagger_i} && \text{with probability at least } 1 - \delta / (4\ell) \\
        &\geq \frac{1}{2^{\ell + 1}} && \text{with probability at least } 1 - \delta / 4 \text{ by induction},
    \end{align*}
    since $\min_{0\le i\le\ell} d_i\ge c\log(4\ell/\delta)$.
    Through a similar argument, we obtain:
    \begin{equation}
        \label{eq:grad_ub}
        \norm*{\tilde{W}_{\ell + 1} \prod_{i = \ell}^1 D_i \tilde{W}_i} \leq 2^{\ell + 1}
    \end{equation}
    with probability at least $1 - \delta / 4$. We also have:
    \begin{gather*}
        \lprp{\tilde{W}_{\ell + 1} \prod_{i = \ell}^1 D_i \tilde{W}_i
        \; \Biggr |\; \norm*{\tilde{W}_{\ell + 1} \prod_{i = \ell}^1
        D_i \tilde{W}_i} = m} \overset{d}{=} \mrm{Unif} \lprp{m
        \mb{S}^{d - 1}}.
    \end{gather*}
    Combining this, \cref{eq:grad_ub,lem:gau_unif_comp}, we get that:
    \begin{equation*}
        \left|f(x)\right| \leq c2^{\ell} \sqrt{\log 1 / \delta}
    \end{equation*}
    with probability at least $1 - \delta / 2$,
    for some absolute constant $c$.
    A union bound over all the preceeding events concludes the lemma.
\end{proof}

\subsection{A Decomposition of Local Gradient Changes}
\label{ssec:grad_decomp}

The following decomposition allows us to reason about deviations layer
by layer:
\begin{align*}
    &\nabla f(x) - \nabla f(y) = W_{\ell + 1} \lprp{\prod_{i = \ell}^1 D_i(\wt{f}_i(x)) W_i - \prod_{i = \ell}^1 D_i(\wt{f}_i(y)) W_i} \\
    &= W_{\ell + 1} \lprp{\prod_{i = \ell}^1 D_i(\wt{f}_i(x)) W_i -
    \left(D_\ell (\wt{f}_\ell (x)) + (D_\ell (\wt{f}_\ell (y)) -
    D_\ell (\wt{f}_\ell (x)))\right) W_\ell\prod_{i = \ell - 1}^1 D_i(\wt{f}_i(y)) W_i} \\
    &= W_{\ell + 1} \left(D_\ell (\wt{f}_\ell (x)) W_\ell \lprp{\prod_{i = \ell - 1}^1 D_i(\wt{f}_i(x)) W_i - \prod_{i = \ell - 1}^1 D_i(\wt{f}_i(y))W_i} \right. \\
    &\qquad\qquad\qquad{} + \left. (D_\ell (\wt{f}_\ell (x)) - D_\ell (\wt{f}_\ell (y))) W_\ell \prod_{i = \ell - 1}^1 D_i(\wt{f}_i(y)) W_i \right) \\
    &= \sum_{j = 1}^\ell \underbrace{W_{\ell + 1} \lprp{\prod_{i = \ell}^{j + 1} D_i(\wt{f}_i(x)) W_i}\cdot (D_j(\wt{f}_j (x)) - D_j(\wt{f}_j (y))) W_j \cdot \lprp{\prod_{i = j - 1}^1 D_i(\wt{f}_i(y)) W_i}}_{\Delta_j} \tag{GD-DECOMP} \label{eq:grad_decomp}.
\end{align*}

We use this decomposition to show that the gradient is locally constant. Concretely, consider a fixed term in the above decomposition and let $i^* = \argmin_{i < j} d_i$. Letting $M_{i^*} = \prod_{i = i^*}^1 D_i (\wt{f}_i (y)) W_i$, we can bound a single term, $\Delta_j$, as follows:
\begin{equation*}
    \norm{\Delta_j} \leq \norm*{W_{\ell + 1} \lprp{\prod_{i = \ell}^{j + 1} D_i(\wt{f}_i(x)) W_i} (D_j(\wt{f}_j (x)) - D_j(\wt{f}_j (y))) W_j \lprp{\prod_{i = j - 1}^{i^* + 1} D_i(\wt{f}_i(y)) W_i}} \norm{M_{i^*}}.
\end{equation*}
We bound the above by bounding each of the two terms in the right-hand-side. In the above expression, the length of $\Delta_j$ is bounded by a product of the length of the corresponding term in a truncated network starting at the output of layer $i^*$ and a product of masked weight matrices up to layer $i^*$ corresponding to the activation patters of $y$. Intuitively, the first term is expected to be small if the images of $x$ and $y$ at layer $i^*$ are close and hence, we show that an image of a suitably small ball around $x$ remains close to the image of $x$ through all the layers of the network (\cref{lem:scale_pres_ml}). We then prove a spectral norm bound on $M_{i^*}$ in \cref{lem:nn_part_spec_bnd} and finally, establish a bound on $\norm{\Delta_j}$ in \cref{lem:gd_error_bnd} where we crucially rely on the scale preservation guarantees provided by \cref{lem:scale_pres_ml}. Finally, combining these results with those of \cref{ssec:f_grad_anticonc} complete the proof of \cref{thm:main}.

\subsection{Scale Preservation of Local Neighborhoods}
\label{ssec:scale_pres_local_nbrhood}

In this section, we show that the image of a ball around $x$
remains in a ball of suitable radius around the image of $x$
projected through the various layers of the network.  Here, we
introduce additional notation used in the rest of the proof:
\begin{gather*}
     \forall j \in \{0\} \cup [\ell + 1]: f_{j, j} (x) = x, \  \forall
     i > j: \wt{f}_{i, j} (x) = W_i f_{i-1, j} (x), \ f_{i, j} (x) =
     D_i(\wt{f}_{i , j}(x)) \wt{f}_{i, j} (x). 
\end{gather*}
We now describe the decomposition of the neural network
into segments, which are bounded by what we call bottleneck
layers, and our analysis works separately with these segments.
This decomposition is crucial for reducing the sizes of the
$\epsilon$-nets that arise in our proofs. Intuitively, when
we construct an $\epsilon$-net to prove that some property
holds uniformly over a ball, it is crucial to work with the
lowest-dimensional image of that ball, which appears in a bottleneck
layer.  These bottleneck layers are denoted by indices $\lbrb{i_j}_{j
= 1}^m$, defined recursively from the output layer backwards with
the convention that $d_0 = d$:
\begin{equation*}
    i_1 \coloneqq \argmin_{i \leq \ell} d_i, \quad \forall j > 1
    \text{ s.t } i_j \geq 1,\, i_{j + 1} \coloneqq \argmin_{j < i_j} d_j, \quad \dmini = \min_{j < i} d_j \tag{NN-DECOMP} \label{eq:nn_dec}.
\end{equation*}
Note, that $i_m = 0$,
and that for all $j \in [m - 1]$, $d_{i_j} < d_{i_{j + 1}}$ and for all
$k \in \lbrb{i_{j+1},\ldots,i_j-1}$, $d_k \geq d_{i_{j+1}}$.

The following technical lemma bounds the scaling of the images at
a layer in the network of an input $x$ and of a ball around $x$.
The crucial properties that we will exploit are that these
images avoid the origin, and that the radius of the image of
the ball is not too large. The full proof is deferred to
\cref{ssec:scale_pres_proof}. In the proof sketch, we carry out the
section of the proof where we transition between bottleneck layers
in full as it is a simple illustration of how such ideas are used
through the rest of the proof.

\begin{lemma}
    \label{lem:scale_pres_ml}
    We have with probability at least $1-\delta$,
    \begin{gather*}
        \forall i \in [\ell]: \norm{f_i(x)} \geq \frac{1}{2^i} \cdot
        \sqrt{d_i},\\
        \forall \norm{x' - x} \leq R: \norm{\wt{f}_i (x) - \wt{f}_i
        (x')} \leq \lprp{C \log \dmax}^{i / 2} \cdot \lprp{1 +
        \sqrt{\frac{d_i}{\dmini}}} \cdot R, \\
        \forall \norm{x' - x} \leq R: \norm{f_i (x) - f_i (x')} \leq
        \lprp{C \log \dmax}^{i / 2} \cdot \lprp{1 +
        \sqrt{\frac{d_i}{\dmini}}} \cdot R.
    \end{gather*}
\end{lemma}
\begin{proof}[Proof Sketch]
    The proof of the first claim is nearly identical to that of \cref{lem:f_grad_conc}. 
    
    For the second claim, we use a gridding argument with some
    subtleties. Concretely, we construct a new grid over the image of
    $\mb{B}(x,R)$ whenever the number of units in a hidden layer drops
    below all the previous layers in the network starting from the input layer. These layers are precisely defined by the indices $i_j$ in \ref{eq:nn_dec}. We now establish the following claim inductively where we adopt the convention $i_0 = \ell + 1$.
    \begin{claim}
        \label{clm:per_segmend_sp}
        Suppose for $j \geq 1$ and $\wt{R} \leq \dmax \cdot R$:
        \begin{equation*}
            \forall y \in \mb{B} (x, R): \norm{f_{i_j}(x) - f_{i_j} (y)} \leq \wt{R}.
        \end{equation*}
        Then:
        \begin{gather*}
            \forall i \in \{i_j + 1, \dots, i_{j - 1} - 1\}, y \in
            \mb{B} (x, R): \norm{\wt{f}_i (x) - \wt{f}_i (y)} \leq \lprp{C \ell \log
            \dmax}^{(i - i_j) / 2} \cdot \sqrt{\frac{d_i}{d_{i_j}}}
            \cdot \wt{R}, \\
            \forall y \in \mb{B} (x, R): \norm{\wt{f}_{i_{j-1}} (x) -
            \wt{f}_{i_{j-1}} (y)}  \cdot \leq \lprp{C \ell \log \dmax}^{(i_{j - 1} - i_j) / 2} \cdot \wt{R}
        \end{gather*}
        with probability at least $1  - \delta / 8l$.
    \end{claim}
    \begin{proof}[Proof Sketch]
        We start by constructing an $\eps$-net \cite[Definition
        4.2.1]{vershynin}, $\mc{G}$, of $f_{i_j} (\mb{B} (x, R))$ with
        $\eps = \epsnetres$. Note that we may assume $\abs{\mc{G}}
        \leq (10 \wt{R} / \eps)^{d_{i_j}}$. We will prove the
        statement on the grid and extend to the rest of the space. For layer $i + 1$, defining $\wt{x} = f_{i_j} (x)$, we have $\forall \wt{y} \in \mc{G}$:
        \begin{align*}
            \norm{\wt{f}_{i + 1, i_j} (\wt{x}) - \wt{f}_{i + 1, i_j} (\wt{y})} &= \norm{W_{i + 1} (f_{i, i_j}(\wt{x}) - f_{i, i_j}(\wt{y}))} \\
            &\leq \norm{f_{i, i_j}(\wt{x}) - f_{i, i_j}(\wt{y})} \cdot \sqrt{\frac{d_{i + 1}}{d_i}} \cdot \lprp{1 + \sqrt{\frac{\log 1 / \delta^\prime}{d_{i + 1}}}}
        \end{align*}
        with probability at least $1 - \delta^\prime$ as before by \cref{thm:tsirelson}. By setting $\delta^\prime = \delta / (16 \abs{\mc{G}} \ell^2)$ and noting that $d_i \geq d_{i_j}$, the conclusion holds for layer $i + 1 \leq i_j$ on $\mc{G}$ with probability at least $1 - \delta / (16\ell^2)$. By induction and the union bound, we get:
        \begin{gather*}
            \forall i \in \{i_j + 1, \dots , i_{j - 1} - 1\}, \wt{y} \in \mc{G}: \norm{\wt{f}_{i, i_j} (\wt{y}) - \wt{f}_{i, i_j}(\wt{x})} \leq (C \ell \log \dmax)^{(i - i_j) / 2} \cdot \sqrt{\frac{d_i}{d_{i_j}}} \cdot \wt{R} \\
            \forall \wt{y} \in \mc{G}: \norm{\wt{f}_{i_{j - 1}, i_j} (\wt{y}) - \wt{f}_{i_{j - 1}, i_j}(\wt{x})} \leq \lprp{C \ell \log \dmax}^{(i_{j - 1} - i_j) / 2} \cdot \wt{R}
        \end{gather*}
        with probability at least $1 - \delta / (16\ell^2)$. To extend
        to all $y \in f_{i_j} (\B (x, R))$, we condition on the bound
        on $\|W_i\|$ given by \cref{lem:gau_spec_conc} for all $i \leq i_{j - 1}$ and note that $\forall y \in f_{i_j} (\B (x, R))$, for $\wt{y} = \argmin_{z \in \mc{G}} \norm{z - y}$, and for $i_j+1\le i<i_{j-1}$,
        \begin{align*}
            \norm{\wt{f}_{i, i_j} (\wt{x}) - \wt{f}_{i, i_j} (y)} &\leq \norm{\wt{f}_{i, i_j} (\wt{x}) - \wt{f}_{i, i_j} (\wt{y})} + \norm{\wt{f}_{i, i_j} (y) - \wt{f}_{i, i_j} (\wt{y})} \\
            &\leq \norm{\wt{f}_{i, i_j} (\wt{x}) - \wt{f}_{i, i_j}
            (\wt{y})} + \norm{y - \wt{y}} \prod_{k = i_{j} + 1}^i
            \norm{W_k} \\
            &\leq \norm{\wt{f}_{i, i_j} (\wt{x}) - \wt{f}_{i, i_j}
            (\wt{y})} + \eps \prod_{k = i_{j} + 1}^i
            \left(C\sqrt{\frac{d_k}{d_{k-1}}}\right) \\
            &= \norm{\wt{f}_{i, i_j} (\wt{x}) - \wt{f}_{i, i_j}
            (\wt{y})} + \eps C^{i-i_j}
            \sqrt{\frac{d_i}{d_{i_j}}},
        \end{align*}
        using that $d_i\ge d_{i_j}$.
        Similarly, for $i=i_{j-1}$, we have
        \begin{align*}
            \norm{\wt{f}_{i_{j-1}, i_j} (\wt{x}) - \wt{f}_{i_{j-1}, i_j} (y)}
            &\leq \norm{\wt{f}_{i_{j-1}, i_j} (\wt{x}) -
            \wt{f}_{i_{j-1}, i_j}
            (\wt{y})} + \eps C^{i_{j-1}-i_j}
            \sqrt{\frac{d_{i_{j-1}-1}}{d_{i_j}}}.
        \end{align*}
        Our setting of $\eps$ concludes the proof of the claim.
    \end{proof}
    An inductive application of \cref{clm:per_segmend_sp}, a union bound and the observation that:
    \begin{equation*}
        \norm{f_i (x) - f_i (y)} \leq \norm{\wt{f}_i (x) - \wt{f}_i (y)}
    \end{equation*}
    concludes the proof of the lemma.
\end{proof}

The following lemma shows that few neurons have inputs of small magnitude for
network input $x$.
\begin{lemma}
    \label{lem:x_i_anticonc}
    With probability at least $1 - \delta$, for all $i \in [\ell]$, we have:
    \begin{equation*}
        \#\lbrb{j: \abs{\inp{(W_{i + 1})_j}{f_i(x)}} \geq \alpha_i
        \frac{\norm{f_i(x)}}{\sqrt{d_i}}} \geq \lprp{1 -
        2\sqrt{\frac{2}{\pi}}\alpha_i} d_{i + 1}.
    \end{equation*}
\end{lemma}
\begin{proof}
    Follows from the fact that $\Pr(|Z|\le c)\le c\sqrt{2/\pi}$ for
    $Z\sim N(0,1)$, plus a simple application of Heoffding's Inequality.
\end{proof}

\subsection{Proving Local Gradient Smoothness}
\label{ssec:local_grad_smooth}

In this section, we show that the gradient is locally constant,
and thus complete the proof of \cref{thm:main}.
In our proof, we will bound each of the terms in the expansion
of the gradient differences \eqref{eq:grad_decomp}.
First, we prove a structural lemma on the
spectral norm of the matrices appearing in the right-hand-side of
\eqref{eq:grad_decomp}, allowing us to ignore the
portion of the network till the last-encountered bottleneck layer.
Define, for all $i > j$,
$M_{i, j} (y) = \prod_{k = i}^{j + 1} D_k (\wt{f}_k(y)) W_k$.
\begin{lemma}\label{lem:nn_part_spec_bnd}
    With probability at least $1 - \delta$ over the $\{W_k\}$, we have:
    \begin{equation*}
        \forall \norm{y - x} \leq R, j \in [m - 1]:
        \P_{\lbrb{D_1(\cdot),\ldots,D_\ell(\cdot)}} \lbrb{\norm*{M_{i_j, i_{j + 1}} (y)} \leq (C
        \cdot \ell \cdot \log \dmax)^{(i_j - i_{j + 1}) / 2}} = 1,
    \end{equation*}
    where the probability is taken with respect to the random choices
    in the definition of the $D_k(\cdot)$.
\end{lemma}
We provide the first part of the proof in full as the simplest application of ideas that find further application in the subsequent result establishing bounds on terms in
\eqref{eq:grad_decomp}; see \cref{ssec:nn_part_spec_bnd_proof}.
\begin{proof}[Proof Sketch]
    To start, consider a fixed $j \in [m]$ and condition on the
    conclusion of \cref{lem:scale_pres_ml} up to level $i_{j + 1}$.
    Now, consider an $\eps$-net of $f_{i_{j + 1}} (\mb{B} (x, R))$,
    $\mc{G}$, with resolution $\eps = \epsnetres$. As before,
    $\abs{\mc{G}} \leq (C\dmax)^{48\ell d_{i_{j + 1}}}$ for some constant $C$. We additionally will consider subsets
    \begin{equation*}
        \mc{S} = \lbrb{(S_k)_{k = i_j - 1}^{i_{j + 1} + 1}:
        S_k\subseteq[d_k],\, \abs{S_k} \leq 4 d_{i_{j + 1}}}.
    \end{equation*}
    Note that $\abs{\mc{G}} \cdot \abs{\mc{S}}^2 \leq (\dmax)^{64ld_{i_{j + 1}}}$. For $y \in \mc{G}, S^1, S^2 \in \mc{S}$, consider the following matrix:
    \begin{equation*}
        M^{i_{j}, i_{j + 1}}_{y, S^1, S^2} = \prod_{k = i_{j}}^{i_{j +
        1} + 1} (D_k (\wt{f}_{k, i_{j+1}} (y))) + (D_{S^1_k} - D_{S^2_k})) W_k \text{ where } 
        (D_S)_{i, j} = 
        \begin{cases}
            1, \text{if } i = j \text{ and } i \in S, \\
            0, \text{otherwise.}
        \end{cases}
    \end{equation*}
    We will bound the spectral norm of $M^{i_{j}, i_{j + 1}}_{y, S^1, S^2}$. First, note that:
    \begin{equation*}
        \norm*{M^{i_j, i_{j + 1}}_{y, S^1, S^2}} \leq
        2 \norm*{\wt{M}^{i_j, i_{j + 1}}_{y, S^1, S^2}} \text{ where }
        \wt{M}^{i_j, i_{j + 1}}_{y, S^1, S^2} \coloneqq W_{i_j}
        \prod_{k = i_{j} - 1}^{i_{j + 1} + 1} (D_k (\wt{f}_{k, i_{j+1}} (y)) + (D_{S^1_k} - D_{S^2_k})) W_k
    \end{equation*}
    and observe that from \cref{lem:dist_equiv}:
    \begin{equation*}
        \wt{M}^{i_{j}, i_{j + 1}}_{y, S^1, S^2} \overset{d}{=} W_{i_j} \prod_{k = i_{j} - 1}^{i_{j + 1} + 1} (D_k + (D_{S^1_k} - D_{S^2_k})) W_k) \text{ where }
        (D_k)_{i,j} = 
        \begin{cases}
            1& \text{w.p } \frac{1}{2} \text{ if } i = j,\\
            0& \text{otherwise.}
        \end{cases}
    \end{equation*}
    To bound the spectral norm, let $\mc{B}$ be a $1/3$-net of $\S^{d_{i_j} - 1}$ and $v \in \mc{B}$. Applying Theorem~\ref{thm:tsirelson},
    \begin{align*}
        \norm*{v^\top \wt{M}^{i_j, i_{j + 1}}_{y, S^1, S^2}} &\leq \norm*{v^\top W_{i_j} \prod_{k = i_j - 1}^{i_{j + 1} + 2} (D_k +  (D_{S^1_k} - D_{S^2_k})) W_k)} \cdot \lprp{1 + \sqrt{\frac{\log 1 / \delta^\prime}{d_{i_{j + 1}}}}} \\
        &\leq \prod_{k = i_{j}}^{i_{j + 1}} \lprp{1 + \sqrt{\frac{\log 1 / \delta^\prime}{d_{k}}}}
    \end{align*}
    with probability at least $\ell \delta^\prime$. But setting
    $\delta^\prime = \delta / (16 \ell^4 \cdot \abs{\mc{G}} \cdot
    \abs{\mc{S}}^2)$ yields with probability at least $1 - \delta /16$:
    \begin{equation*}
        \forall y \in \mc{G}, S_1, S_2 \in \mc{S}:
         \norm*{\wt{M}^{i_j, i_{j + 1}}_{y, S^1, S^2}} \leq (C \cdot \ell \cdot \log \dmax)^{(i_j - i_{j + 1}) / 2}.
    \end{equation*}
    On the event in the conclusion of \cref{lem:scale_pres_ml}, we have that $f_i (y) \neq 0$ for all $i \in [\ell], y \in \mc{G}$ and therefore, we have by a union bound over the discrete set $\mc{G}$:
    \begin{equation*}
        \forall y \in \mc{G}, k \in \lbrb{i_{j + 1}, \dots, i_j}, m \in [d_k]: (\wt{f}_{k, i_j} (y))_m \neq 0
    \end{equation*}
    proving the lemma for $y \in \mc{G}$ as the activations are deterministic. For $y \notin \mc{G}$, the following claim concludes the proof of the lemma. The claim is essentially a generalization of \cite[Eq. (18)]{bubeck21} and its proof is deferred to the appendix.
    \begin{claim}
        \label{clm:grid_approx_act_patt_spec_bnd}
        With probability at least $1 - \delta^\prime / \ell^2$ over
        the $\{W_k\}$,
        we have for all $m \in \{i_{j + 1}+1, \dots, i_{j}\}$ and $y \in
        f_{i_{j+1}} (\mb{B} (x, R))$:
        \begin{equation*}
            \P_{\lbrb{D_k (\wt{f}_{k,i_{j+1}} (y)), D_k
            (\wt{f}_{k,i_{j+1}} (\wt{y}))}} \lbrb{\Tr \abs{D_m (\wt{f}_{m, i_{j + 1}} (y)) - D_m (\wt{f}_{m, i_{j + 1}} (\wt{y}))} \leq 4d_{i_{j + 1}}} = 1
        \end{equation*}
        where $\wt{y} = \argmin_{z \in \mc{G}} \norm{z - y}$, and
    the probability is taken with respect to the random choices
    in the definition of the $D_k(\cdot)$.
    \end{claim}
    The previous claim along with our previously established bounds establish the lemma.
\end{proof}

Our final technical result establishes the near-linearity of $f$ in a ball around $x$. The full proof is deferred to \cref{ssec:gd_error_bnd_proof} but in our proof sketch we identify sections which involve considerations unique to this lemma.
\begin{lemma}
    \label{lem:gd_error_bnd}
    For some absolute constant $C$,
    with probability at least $1 - \delta$ over the $\{W_k\}$:
    \begin{equation*}
        \forall \norm{x - y} \leq R, j \in [\ell]: \P_{\lbrb{D_k
        (\wt{f}_k (x)), D_k (\wt{f}_k (y))}} \lbrb{\norm{\nabla f(x) -
        \nabla f(y)} \leq \lprp{\frac{C^\ell}{\log^\ell \dmax}}} = 1,
    \end{equation*}
    where the probability is taken with respect to the random choices
    in the definition of the $D_k(\cdot)$.
\end{lemma}
\begin{proof}[Proof Sketch]
    Consider a fixed term from \eqref{eq:grad_decomp}; that is, consider:
    \begin{equation*}
        \gddiff_j(y) \coloneqq W_{\ell + 1} \lprp{\prod_{i = \ell}^{j + 1} D_i(\wt{f}_i(x)) W_i} \cdot (D_j(\wt{f}_j (y)) - D_j(\wt{f}_j (x))) W_j \cdot \lprp{\prod_{i = j - 1}^1 D_i(\wt{f}_i(y)) W_i}.
    \end{equation*}
    We will show with high probability that $\gddiff_j(y)$ is small
    for all $y\in \B (x, R)$. This will then imply the lemma by a
    union bound and \eqref{eq:grad_decomp}. Let $k$ be such that $i_k = \argmin_{m < j} d_m$. We will condition on the weights of the network up to layer $i_k$. Specifically, we will assume the conclusions of \cref{lem:scale_pres_ml,lem:nn_part_spec_bnd} up to layer $i_k$. We may now focus our attention solely on the segment of the network beyond layer $i_k$ as a consequence of the following observation and \cref{lem:nn_part_spec_bnd}:
    \begin{gather*}
        \norm*{\gddiff_j(y)} \leq \norm{\mrm{Diff}_{j,k} (x,y)} \cdot \norm*{M_{i_k, 0} (y)} \text{ where } \tag{\theequation} \stepcounter{equation} \label{eq:gd_final_decomp} \\
        \mrm{Diff}_{j,k} (x,y) \coloneqq W_{\ell + 1} \lprp{\prod_{i =
        \ell}^{j + 1} D_i(\wt{f}_i(x)) W_i}  (D_j(\wt{f}_j (y)) -
        D_j(\wt{f}_j (x))) W_j  \lprp{\prod_{i = j - 1}^{i_k+1} D_i(\wt{f}_i(y)) W_i} 
    \end{gather*}
    We will show for all $y$ such that $\norm{y - x} \leq R$:
    \begin{equation}
        \label{eq:gd_part_bnd}
        \P_{\lbrb{D_m (\wt{f}_m (x)), D_m (\wt{f}_m (y))}} \lbrb{\norm*{\mrm{Diff}_{j,k} (x,y)} \geq \frac{C^\ell}{(\ell \log \dmax)^{3\ell}}} = 0
    \end{equation}
    with probability at least $1 - \delta / (16\ell^2)$. 
    
    We have from \cref{lem:dist_equiv} that the random vector $H$ defined below is spherically symmetric and satisfies $\norm{H} \leq 2^\ell$ with probability at least $1 - \delta / (16\ell^4)$:
    \begin{gather*}
        W_{\ell + 1} \prod_{i = \ell}^{j + 1} D_i(\wt{f}_i(x)) W_i \overset{d}{=} \wt{W}_{\ell + 1} \prod_{i = \ell}^{j + 1} D_i \wt{W}_i \eqqcolon H.
    \end{gather*}
    As in the proof of \cref{lem:nn_part_spec_bnd}, let $\mc{G}$ be an $\eps$-net of $f_{i_k}(\mb{B} (x, R))$ with $\eps$ as in the proof of \cref{lem:nn_part_spec_bnd}. We now break into two cases depending on how $d_{j}$ compares to $d_{i_k}$ and handle them separately. At this point, we have effectively reduced the multi-layer proof to the problem of analyzing deviations of activations at a fixed layer. For the remaining proof, we generalize approaches in \cite{daniely20} for the small width case and \cite{bubeck21} for the large width case.
    
    \paragraph{Case 1:} $d_j \leq d_{i_k} (\ell \log \dmax)^{20\ell}$. In this case, the key observation already made in \cite{daniely20} is that under \cref{lem:scale_pres_ml}, the number of neurons that may actually differ at layer $j$ is at most $d_j / \poly (\ell \log \dmax)^\ell$ between $y \in \B (x, R)$ and $x$. Since, $H$ is spherically distributed, $\norm{H (D_j(\wt{f}_j (x)) - D_j(\wt{f}_j (y)))}$ is very small and consequently the whole term is small.

    \paragraph{Case 2:} $d_j \geq d_{i_k} (\ell \log \dmax)^{20\ell}$.
    This case is analogous to \cite{bubeck21}. The proof is
    technically involved and requires careful analysis of the random
    vector $H (D_j(\wt{f}_j (x)) - D_j(\wt{f}_j (y))) W_j$, which is
    complicated in our setting due to the matrices preceeding it in
    \eqref{eq:grad_decomp}. It involves the distributional
    equivalence in Lemma~\ref{lem:dist_equiv}.

    A union bound and an application of the triangle inequality now imply the lemma.
\end{proof}

\begin{proof}[Proof of \cref{thm:main}]
    On the intersection of the events in the conclusions of \cref{lem:scale_pres_ml,lem:f_grad_conc,lem:gd_error_bnd}, we have that $\nabla f(x)$ is deterministic and furthermore, we have:
    \begin{gather*}
        \norm{\nabla f(x)} \geq \frac{1}{2^{\ell + 1}},\quad \abs{f(x)} \leq c2^{\ell} \sqrt{\log 1 / \delta} \\
        \forall y \text{ s.t }\norm{y - x} \leq R: \norm{\nabla f(x) - \nabla f(y)} \leq \frac{1}{(\ell \log \dmax)^\ell} = o(1).
    \end{gather*}
    Assume $f(x) > 0$ (the alternative is similar) and let $\eta = -
    \frac{2^{\ell}\log d \cdot \sqrt{\log 1 / \delta}}{\norm{\nabla f(x)}^2}$, we have for the point $x + \eta \nabla f(x)$, defining the function $g(t) = f(x + t \eta \nabla f(x))$:
    \begin{align*}
        f(x + \eta \nabla f(x)) &= f(x) + \int_{0}^1 g' (t) dt = f(x) + \int_0^1 (\eta \nabla f(x))^\top \nabla f(x + t \eta \nabla f(x)) dt \\
        &= f(x) - (1 - o(1)) 2^{\ell}\log d \sqrt{\log 1 / \delta} \leq -f(x).
    \end{align*}
    Our lower bounds on $\nabla f(x)$ ensure $\norm{\eta \nabla f(x)}
    \leq R$, concluding the proof of the theorem.
\end{proof}

%% file: content/lower.tex
\section{The Impact of Depth}
\label{sec:lower}

\cref{thm:main} relies on the depth being constant.
In this section, we show that some constraint on the depth
is necessary in order to ensure the existence of adversarial
examples. In particular, the following result gives an example of
a sufficiently deep network for which, with high probability, the
output of the network will have the same sign for all input patterns.

\begin{theorem}\label{thm:lower-in-body}
  Fix a sufficiently large $d\in\N$, an $\ell \geq d^3$ and
  $(\ell d)^{20} \leq k \leq \exp (\sqrt{\ell})$, and consider
  the randomly initialized neural network~\eqref{eq:nn_def} with
  $d_1=\cdots=d_\ell=k$.  There is a universal constant $C$ such
  that with probability at least $0.9$,
  \begin{gather*}
    \forall x,y \in \S^{d - 1}: \frac{\abs{f(x) - f(y)}}{\abs{f(x)}}
    \leq C \sqrt{\frac{\log d}{d}}.
  \end{gather*}
\end{theorem}

The proof involves showing that the image of $\S^{d-1}$ is bounded
away from $0$, and that the inner products between images of two
input vectors throughout the network converge. The following lemma
shows how the expected inner products evolve through the network.
The lemma follows from computing a double integral; see
\cite[Eq.~(6)]{NIPS2009_5751ec3e}. The proof of
Theorem~\ref{thm:lower-in-body} is in \cref{sec:proof_lower}.

\begin{lemma}
    \label{lem:expectedinnerprod}
    Let $d \in \N$. Fix $x,y\in\mb{R}^d$ with $x,y \neq 0$. Then for $g \thicksim \mc{N} (0, I)$:
    \begin{equation*}
      \frac{\E \lsrs{\max \lbrb{g^\top x, 0} \max \lbrb{g^\top y, 0}}}{\sqrt{\E (\max \lbrb{g^\top x, 0})^2 \cdot \E (\max \lbrb{g^\top y, 0})^2}} = \frac{\sin \theta}{\pi} + \lprp{1 - \frac{\theta}{\pi}} \cos \theta
    \end{equation*}
    where $ \theta = \arccos (x^\top y/(\norm{x} \norm{y}))$.
\end{lemma}

%% file: content/appendix.tex
\input{content/def_proofs}
\input{content/lower_bound_app}
\input{content/misc}

%% file: content/def_proofs.tex
\section{Deferred Proofs from \cref{sec:proof_relu}}
\label{sec:def_proofs}

\subsection{Proof of \cref{lem:dist_equiv}}
\label{ssec:dis_equiv_proof}

For the first claim, we introduce the following diagonal
random signed matrices, $S_i$ for $i \in [m]$ with:
\begin{equation*}
    (S_i)_{j,j} = \pm 1 \text{ with equal probability },\quad
    (S_i)_{j,k} = 0 \text{ if $j\not=k$}.
\end{equation*}
We prove the claim by induction on $m$. When $m = 1$, the claim is trivially true. For $m > 1$, we have:
\begin{align*}
    \lefteqn{W_m \prod_{j = m - 1}^1  (D_j (\wt{h}_j (x)) + B_j) W_j}
    & \\
    &\overset{d}{=} W_m \lprp{\prod_{j = m - 1}^2  (D_j (\wt{h}_j (x)) + B_j) W_j} (D_1 (S_1 \wt{h}_1 (x)) + B_1) S_1 W_1 
    & &\left(S_1W_1 \overset{d}{=} W_1\right) \\
    &\overset{d}{=} W_m \lprp{\prod_{j = m - 1}^2  (D_j + B_j) W_j} (D_1 (S_1 \wt{h}_1 (x)) + B_1) S_1 W_1 
    & &\left(\text{hypothesis for $m-1$}\right) \\
    &\overset{d}{=} W_m \lprp{\prod_{j = m - 1}^2  (D_j + B_j) W_j} S_1 (D_1(S_1 \wt{h}_1 (x)) + B_1)  S_1 W_1
    & &\left(W_2S_1 \overset{d}{=} W_2\right) \\
    &= W_m \lprp{\prod_{j = m - 1}^2  (D_j + B_j) W_j} (D_1 (S_1 \wt{h}_1 (x)) + B_1) W_1 
    & &\left(S_1B_1S_1 = B_1\right) \\
    &\overset{d}{=} W_m \lprp{\prod_{j = m - 1}^2  (D_j + B_j) W_j} (D_1 + B_1) W_1 \\
    &\overset{d}{=} W_m \prod_{j = m - 1}^1  (D_j + B_j) W_j.
\end{align*}

Similarly, for the second claim, we have:
\begin{align*}
    &\norm*{\prod_{j = m}^1  (D_j (\wt{h}_j (x)) + B_j) W_j} \\
    &\overset{d}{=} \norm*{(D_m (S_m\wt{h}_m (x)) + B_m) S_m W_m \prod_{j = m - 1}^1  (D_j (\wt{h}_j (x)) + B_j) W_j} \\
    & \overset{d}{=} \norm*{S_m (D_m (S_m\wt{h}_m (x)) + B_m) S_m W_m \prod_{j = m - 1}^1  (D_j (\wt{h}_j (x)) + B_j) W_j} \\
    & \overset{d}{=} \norm*{(D_m (S_m\wt{h}_m (x)) + B_m) W_m \prod_{j = m - 1}^1  (D_j (\wt{h}_j (x)) + B_j) W_j} \\
    & \overset{d}{=} \norm*{(D_m + B_m) W_m \prod_{j = m - 1}^1  (D_j (\wt{h}_j (x)) + B_j) W_j} \\
    &\overset{d}{=} \norm*{\prod_{j = m}^1  (D_j + B_j) W_j}
\end{align*}
\qed

\subsection{Proof of \cref{lem:scale_pres_ml}}
\label{ssec:scale_pres_proof}

We start with the first claim of the lemma. As in \cref{lem:f_grad_conc}, we note from \cref{lem:dist_equiv}:
\begin{gather*}
    f_i (x) = \left(\prod_{j = i}^1 D_j(\wt{f}_j (x)) W_j\right) x\\
    \norm*{\prod_{j = i}^1 D_j(\wt{f}_j (x)) W_j} \overset{d}{=} \norm*{\prod_{j = i}^1 D_j \wt{W}_j} \text{ where } \\
    (D_i)_{j,k} = \begin{cases}
        1 & \text{w.p } 1/2 \text{ if } j = k\\
        0 & \text{otherwise}
        \end{cases}
\text{ and } \{W_i\}_{i = 1}^{\ell + 1} \overset{d}{=} \{\wt{W}_i\}_{i = 1}^{\ell + 1}.
\end{gather*}
Therefore, it suffices to analyze the distribution of $\prod_{j = i}^1 D_j \wt{W}_j x$. Again, we condition on the following event:
\begin{equation*}
    \forall i \in [\ell]: \Tr D_i \geq \frac{d_i}{3}
\end{equation*}
which occurs with probability at least $1 - \delta / 8$. As in the
proof of \cref{lem:f_grad_conc}, we observe the following
distributional equivalence, when we condition on the $D_j$:
\begin{gather*}
    \forall i \in [\ell]: \norm*{\prod_{j = i}^1 D_j \wt{W}_j x}
    \overset{d}{=} \norm*{\prod_{j = i}^1 W^\dagger_j x} \text{
    where } \\
    \forall j \in \{2, \dots, i\},\,
    W^\dagger_j \in \R^{\Tr D_j \times \Tr D_{j - 1}},\ W^\dagger_1 \in \R^{\Tr D_1 \times d} \\
    \forall j \in \{2, \dots, i\},\,
    (W^\dagger_j)_{k, :} \thicksim \mc{N} \lprp{0, I / d_{j -
    1}},\, (W^\dagger_1)_{k, :} \thicksim \mc{N} (0, I / d).
\end{gather*}
A recursive application of \cref{thm:tsirelson} as in the proof of \cref{lem:f_grad_conc} yields:
\begin{equation*}
    \norm*{\prod_{j = i}^1 D_j \wt{W}_j x}
    \geq \|x\| \prod_{j = i}^1 \frac{1}{2}\sqrt{\frac{d_j}{d_{j-1}}}
\end{equation*}
with probability at least $1 - \delta / 16\ell^2$. A union bound over the $j$ layers concludes the proof of the first claim of the lemma.

We will establish the second claim through a gridding based argument with some subtleties. Namely, we construct a new grid over the image of $\mb{B}(x,R)$ whenever the number of units in a hidden layer drops lower than all the previous layers in the network starting from the input layer. These layers are precisely defined by the indices $i_j$ in \ref{eq:nn_dec}. We now establish the following claim inductively where we adopt the convention $i_0 = \ell + 1$.
\begin{claim}
    \label{clm:per_segmend_sp_app}
    Suppose for $j \geq 1$ and $\wt{R} \leq \dmax \cdot R$:
    \begin{equation*}
        \forall y \in \mb{B} (x, R): \norm{f_{i_j}(x) - f_{i_j} (y)} \leq \wt{R}.
    \end{equation*}
    Then:
    \begin{gather*}
        \forall i \in \{i_j + 1, \dots, i_{j - 1} - 1\}, y \in
        \mb{B} (x, R): \norm{\wt{f}_i (x) - \wt{f}_i (y)} \leq \lprp{C \ell \log
        \dmax}^{(i - i_j) / 2} \cdot \sqrt{\frac{d_i}{d_{i_j}}}
        \cdot \wt{R}, \\
        \forall y \in \mb{B} (x, R): \norm{\wt{f}_{i_{j-1}} (x) - \wt{f}_{i_{j-1}} (y)}  \cdot \leq \lprp{C \ell \log \dmax}^{(i_{j - 1} - i_j) / 2} \cdot \wt{R}
    \end{gather*}
    with probability at least $1  - \delta / 8l$.
\end{claim}
\begin{proof}
    We start by constructing an $\eps$-net \cite[Definition
    4.2.1]{vershynin}, $\mc{G}$, of $f_{i_j} (\mb{B} (x, R))$ with
    $\eps = \epsnetres$. Note that we may assume $\abs{\mc{G}} \leq
    (10 \wt{R} / \eps)^{d_{i_j}}$. We will prove the statement on the
    grid and extend to the rest of the space. For layer $i + 1$, defining $\wt{x} = f_{i_j} (x)$, we have $\forall \wt{y} \in \mc{G}$:
    \begin{align*}
        \norm{\wt{f}_{i + 1, i_j} (\wt{x}) - \wt{f}_{i + 1, i_j} (\wt{y})} &= \norm{W_{i + 1} (f_{i, i_j}(\wt{x}) - f_{i, i_j}(\wt{y}))} \\
        &\leq \norm{f_{i, i_j}(\wt{x}) - f_{i, i_j}(\wt{y})} \cdot \sqrt{\frac{d_{i + 1}}{d_i}} \cdot \lprp{1 + \sqrt{\frac{\log 1 / \delta^\prime}{d_{i + 1}}}}
    \end{align*}
    with probability at least $1 - \delta^\prime$ as before by \cref{thm:tsirelson}. By setting $\delta^\prime = \delta / (16 \abs{\mc{G}} \ell^2)$ and noting that $d_i \geq d_{i_j}$, the conclusion holds for layer $i + 1 \leq i_j$ on $\mc{G}$ with probability at least $1 - \delta / (16\ell^2)$. By induction and the union bound, we get:
    \begin{gather*}
        \forall i \in \{i_j + 1, \dots , i_{j - 1} - 1\}, \wt{y} \in \mc{G}: \norm{\wt{f}_{i, i_j} (\wt{y}) - \wt{f}_{i, i_j}(\wt{x})} \leq (C \ell \log \dmax)^{(i - i_j) / 2} \cdot \sqrt{\frac{d_i}{d_{i_j}}} \cdot \wt{R} \\
        \forall \wt{y} \in \mc{G}: \norm{\wt{f}_{i_{j - 1}, i_j} (\wt{y}) - \wt{f}_{i_{j - 1}, i_j}(\wt{x})} \leq \lprp{C \ell \log \dmax}^{(i_{j - 1} - i_j) / 2} \cdot \wt{R}
    \end{gather*}
    with probability at least $1 - \delta / (16\ell^2)$. To extend to
    all $y \in f_{i_j} (\B (x, R))$, we condition on the bound on
    $\|W_i\|$ given by \cref{lem:gau_spec_conc} for all $i \leq i_{j - 1}$ and note that $\forall y \in f_{i_j} (\B (x, R))$, for $\wt{y} = \argmin_{z \in \mc{G}} \norm{z - y}$, and for $i_j+1\le i<i_{j-1}$,
    \begin{align*}
        \norm{\wt{f}_{i, i_j} (\wt{x}) - \wt{f}_{i, i_j} (y)} &\leq \norm{\wt{f}_{i, i_j} (\wt{x}) - \wt{f}_{i, i_j} (\wt{y})} + \norm{\wt{f}_{i, i_j} (y) - \wt{f}_{i, i_j} (\wt{y})} \\
        &\leq \norm{\wt{f}_{i, i_j} (\wt{x}) - \wt{f}_{i, i_j}
        (\wt{y})} + \norm{y - \wt{y}} \prod_{k = i_{j} + 1}^i
        \norm{W_k} \\
        &\leq \norm{\wt{f}_{i, i_j} (\wt{x}) - \wt{f}_{i, i_j}
        (\wt{y})} + \eps \prod_{k = i_{j} + 1}^i
        \left(C\sqrt{\frac{d_k}{d_{k-1}}}\right) \\
        &= \norm{\wt{f}_{i, i_j} (\wt{x}) - \wt{f}_{i, i_j}
        (\wt{y})} + \eps C^{i-i_j}
        \sqrt{\frac{d_i}{d_{i_j}}},
    \end{align*}
    using that $d_i\ge d_{i_j}$.
    Similarly, for $i=i_{j-1}$, we have
    \begin{align*}
        \norm{\wt{f}_{i_{j-1}, i_j} (\wt{x}) - \wt{f}_{i_{j-1}, i_j} (y)}
        &\leq \norm{\wt{f}_{i_{j-1}, i_j} (\wt{x}) -
        \wt{f}_{i_{j-1}, i_j}
        (\wt{y})} + \eps C^{i_{j-1}-i_j}
        \sqrt{\frac{d_{i_{j-1}-1}}{d_{i_j}}}.
    \end{align*}
    Our setting of $\eps$ concludes the proof of the claim.
\end{proof}
An inductive application of \cref{clm:per_segmend_sp_app}, a union bound and the observation that:
\begin{equation*}
    \norm{f_i (x) - f_i (y)} \leq \norm{\wt{f}_i (x) - \wt{f}_i (y)}
\end{equation*}
concludes the proof of the lemma.
\qed

\subsection{Proof of \cref{lem:nn_part_spec_bnd}}
\label{ssec:nn_part_spec_bnd_proof}
To start, consider a fixed $j \in [m]$ and condition on the conclusion of
\cref{lem:scale_pres_ml} up to level $i_{j + 1}$. Now, consider an $\eps$-net of
$f_{i_{j + 1}} (\mb{B} (x, R))$, $\mc{G}$, with resolution $\eps = \epsnetres$. As
before, $\abs{\mc{G}} \leq (C\dmax)^{48\ell d_{i_{j + 1}}}$ for some constant $C$. We additionally will consider subsets
\begin{equation*}
    \mc{S} = \lbrb{(S_k)_{k = i_j - 1}^{i_{j + 1} + 1}:
    S_k\subseteq[d_k],\, \abs{S_k} \leq 4 d_{i_{j + 1}}}.
\end{equation*}
Note that $\abs{\mc{G}} \cdot \abs{\mc{S}}^2 \leq (C\dmax)^{64ld_{i_{j + 1}}}$. For $y \in \mc{G}, S^1, S^2 \in \mc{S}$, consider the following matrix:
\begin{gather*}
    M^{i_{j}, i_{j + 1}}_{y, S^1, S^2} = \prod_{k = i_{j}}^{i_{j + 1}
    + 1} (D_k (\wt{f}_{k, i_{j+1}} (y))) + (D_{S^1_k} - D_{S^2_k})) W_k \text{ where } \\
    (D_S)_{i, j} = 
    \begin{cases}
        1, \text{if } i = j \text{ and } i \in S \\
        0, \text{otherwise}
    \end{cases}.
\end{gather*}
We will bound the spectral norm of $M^{i_{j}, i_{j + 1}}_{y, S^1, S^2}$. First, note that:
\begin{equation*}
    \norm*{M^{i_j, i_{j + 1}}_{y, S^1, S^2}} \leq 2 \norm*{\wt{M}^{i_j,
    i_{j + 1}}_{y, S^1, S^2}} \text{ where } \wt{M}^{i_j, i_{j +
    1}}_{y, S^1, S^2} \coloneqq W_{i_j} \prod_{k = i_{j} - 1}^{i_{j +
    1} + 1} (D_k (\wt{f}_{k, i_{j+1}} (y)) + (D_{S^1_k} - D_{S^2_k})) W_k
\end{equation*}
and observe that from \cref{lem:dist_equiv}:
\begin{gather*}
    \wt{M}^{i_{j}, i_{j + 1}}_{y, S^1, S^2} \overset{d}{=} W_{i_j} \prod_{k = i_{j} - 1}^{i_{j + 1} + 1} (D_k + (D_{S^1_k} - D_{S^2_k})) W_k) \text{ where } \\
    (D_k)_{i,j} = 
    \begin{cases}
        1 & \text{with probability } \frac{1}{2} \text{ if } i = j,\\
        0 & \text{otherwise.}
    \end{cases}
\end{gather*}
To bound the spectral norm, let $\mc{B}$ be a $1/3$-net of $\S^{d_{i_j} - 1}$ and
$v \in \mc{B}$. Applying Theorem~\ref{thm:tsirelson},
\begin{align*}
    \norm*{v^\top \wt{M}^{i_j, i_{j + 1}}_{y, S^1, S^2}} &\leq \norm*{v^\top W_{i_j} \prod_{k = i_j - 1}^{i_{j + 1} + 2} (D_k +  (D_{S^1_k} - D_{S^2_k})) W_k)} \cdot \lprp{1 + \sqrt{\frac{\log 1 / \delta^\prime}{d_{i_{j + 1}}}}} \\
    &\leq \prod_{k = i_{j}}^{i_{j + 1}} \lprp{1 + \sqrt{\frac{\log 1 / \delta^\prime}{d_{k}}}}
\end{align*}
with probability at least $\ell \delta^\prime$. But setting $\delta^\prime =
\delta / (16 \ell^4 \cdot \abs{\mc{G}} \cdot \abs{\mc{S}}^2)$ yields
  \begin{align*}
    \frac{\log 1/\delta'}{d_k}
      &\le \frac{C\ell d_{i_{j+1}}\log\dmax}{d_k} + \frac{\log 1/\delta}{d_k} \\
      &\le C'\ell \log\dmax
  \end{align*}
because $d_k\ge d_{i_{j+1}}$ and by the assumption on $\delta$.
Hence, with probability at least $1 - \delta /16$:
    \begin{equation*}
        \forall y \in \mc{G}, S_1, S_2 \in \mc{S}:
         \norm*{\wt{M}^{i_j, i_{j + 1}}_{y, S^1, S^2}} \leq (C \cdot \ell \cdot \log \dmax)^{(i_j - i_{j + 1}) / 2}.
    \end{equation*}
On the event in the conclusion of \cref{lem:scale_pres_ml}, we have that
for all $y\in\mb{B}(x,R)$ and $i\in[\ell]$, $f_i (y) \neq 0$, and therefore we
have by a union bound over the discrete set $\mc{G}$ in an event of probability
$1$,
\begin{equation*}
    \forall y \in \mc{G}, k \in \lbrb{i_{j + 1}, \dots, i_j}, m \in
    [d_k]: (\wt{f}_{k, i_j} (y))_m \neq 0.
\end{equation*}
We now show a basic structural claim of how activation patterns differ across the various layers between a point and its closest approximation in $\mc{G}$:
\begin{claim}
    \label{clm:grid_approx_act_patt_spec_bnd_app}
    With probability at least $1 - \delta^\prime / \ell^2$ over the $W_k$,
    we have for all $m \in \{i_{j + 1} + 1, \dots, i_{j}\}$ and
    $y \in f_{i_{j+1}} (\mb{B} (x, R))$:
    \begin{equation*}
        \P_{\lbrb{D_n (\wt{f}_{n,i_{j+1}} (y)), D_n (\wt{f}_{n,i_{j+1}}
        (\wt{y}))}} \lbrb{\Tr \abs{D_m (\wt{f}_{m, i_{j + 1}} (y)) - D_m
        (\wt{f}_{m, i_{j + 1}} (\wt{y}))} \leq 4d_{i_{j + 1}}} = 1,
    \end{equation*}
    where $\wt{y} = \argmin_{z \in \mc{G}} \norm{z - y}$.
\end{claim}
\begin{proof}
    We start by conditioning on the conclusion of \cref{lem:gau_spec_conc} (with
    $A=\sqrt{d_{k-1}}W_k$) up to layer $i_{j + 1}$. Then:
    \begin{equation*}
        \forall m \in \{i_{j + 1} + 1, \cdots, i_{j}\}, y \in f_{i_{j+1}} (\mb{B} (x, R)): \norm{\wt{f}_{m, i_{j + 1}} (y) - \wt{f}_{m, i_{j + 1}} (\wt{y})} \leq \eps \cdot \lprp{C \dmax}^{(m - i_{j + 1}) / 2}.
    \end{equation*}
    Now, let $m \in \{i_{j + 1}+1, \cdots, i_{j}\}$ and $y \in \mc{G}$. We now operate on the conclusion of \cref{lem:scale_pres_ml} up to layer $m - 1$ and hence, $(\wt{f}_{m, i_{j + 1}} (y))_n \neq 0$ for all $n$. Now, we have:
    \begin{multline*}
        \max_{z \text{ s.t } \norm{z - y} \leq \eps} \# \{i: \sign (\inp{(W_m)_{i,
        :}}{f_{m - 1, i_{j+1}} (y)}) \neq \sign (\inp{(W_m)_{i, :}}{f_{m - 1,
        i_{j+1}} (z)})\} \\
        \leq \# \{i: \exists z \text{ s.t } \norm{z - y} \leq \eps \text{ and }
        \sign (\inp{(W_m)_{i, :}}{f_{m - 1, i_{j+1}} (y)}) \neq \sign
        (\inp{(W_m)_{i, :}}{f_{m - 1, i_{j+1}} (z)})\}.
    \end{multline*}
    Defining $w_i:=(W_m)_{i, :}$ and
    \begin{equation*}
        W_i(y) \coloneqq \bm{1} \{\exists z \text{ s.t } \norm{z - y} \leq \eps
        \text{ and } \sign (\inp{(W_m)_{i, :}}{f_{m - 1, i_{j+1}} (y)}) \neq \sign
        (\inp{(W_m)_{i, :}}{f_{m - 1, i_{j+1}} (z)})\},
    \end{equation*}
    we get:
    \begin{align*}
        \P \lbrb{W_i(y) = 1} &\leq \P  \lbrb{(C \dmax)^{(m -1 - i_{j+1}) / 2} \cdot \eps \cdot \norm{w_i} \geq \abs{\inp{w_i}{y}}} \\
        &\leq \P \lbrb{\norm{w_i} \geq 2} + \P \lbrb{2\eps (C \dmax)^{(m -1 -
        i_{j+1}) / 2} \geq \abs{\inp{w_i}{y}}} \leq \eps \cdot (C \dmax)^{(m - i_{j+1}) / 2}.
    \end{align*}
    Therefore, we have:
    \begin{equation*}
        \P \lbrb{\sum_{i = 1}^{d_m} W_i(y) \geq 4d_{i_{j+1}}} \leq d_m^{4d_{i_{j + 1}}} \cdot \lprp{\eps (C \dmax)^{(m - i_{j+1}) / 2}}^{4d_{i_{j + 1}}}
    \end{equation*}
    and hence by the union bound and our setting of $\eps$ and bound on $\mc{G}$:
    \begin{equation*}
        \P \lbrb{\exists y \in \mc{G}: \sum_{i = 1}^{d_m} W_i (y) \geq 4d_{i_{j + 1}}} \leq (C\dmax)^{60 \ell d_{i_{j + 1}}} \cdot \eps^{4 d_{i_{j + 1}}} \leq \frac{\delta'}{16\ell^4}.
    \end{equation*}
\end{proof}
The claim concludes the proof of the lemma.
\qed

\subsection{Proof of \cref{lem:gd_error_bnd}}
\label{ssec:gd_error_bnd_proof}

    Consider a fixed term in the decomposition of the gradient difference from \ref{eq:grad_decomp}; that is, consider the random vector-valued function:
    \begin{equation*}
        \gddiff_j(y) \coloneqq W_{\ell + 1} \lprp{\prod_{i = \ell}^{j + 1} D_i(\wt{f}_i(x)) W_i} \cdot (D_j(\wt{f}_j (y)) - D_j(\wt{f}_j (x))) W_j \cdot \lprp{\prod_{i = j - 1}^1 D_i(\wt{f}_i(y)) W_i}.
    \end{equation*}
    We will show with high probability that $\gddiff_j(y) = o(1)$ with high probability for all $\norm{x - y} \leq R$. This will then imply the lemma by a union bound and \ref{eq:grad_decomp}. Let $k$ be such that $i_k = \argmin_{m < j} d_m$. We will condition on the weights of the network up to layer $i_k$. Specifically, we will assume the conclusions of \cref{lem:scale_pres_ml,lem:nn_part_spec_bnd} up to layer $i_k$. We may now focus our attention solely on the segment of the network beyond layer $i_k$ as a consequence of the following observations:
    \begin{gather*}
        \norm*{\gddiff_j(y)} \leq \norm{\mrm{Diff}_{j,k} (x,y)} \cdot \norm*{M_{i_k, 0} (y)} \text{ where } \tag{\theequation} \stepcounter{equation} \label{eq:gd_final_decomp} \\
        \mrm{Diff}_{j,k} (x,y) \coloneqq W_{\ell + 1} \lprp{\prod_{i = \ell}^{j +
        1} D_i(\wt{f}_i(x)) W_i} (D_j(\wt{f}_j (y)) - D_j(\wt{f}_j (x))) W_j
        \lprp{\prod_{i = j - 1}^{i_k+1} D_i(\wt{f}_i(y)) W_i} 
    \end{gather*}
    We will show for all $y$ such that $\norm{y - x} \leq R$:
    \begin{equation}
        \label{eq:gd_part_bnd}
        \P_{\lbrb{D_m (\wt{f}_m (x)), D_m (\wt{f}_m (y))}} \lbrb{\norm*{\mrm{Diff}_{j,k} (x,y)} \geq \frac{C^\ell}{(\ell \log \dmax)^{3\ell}}} = 0
    \end{equation}
    with probability at least $1 - \delta / (16\ell^2)$. 
    
    Observe now that from \cref{lem:dist_equiv}:
    \begin{gather*}
        W_{\ell + 1} \prod_{i = \ell}^{j + 1} D_i(\wt{f}_i(x)) W_i \overset{d}{=} \wt{W}_{\ell + 1} \prod_{i = \ell}^{j + 1} D_i \wt{W}_i \eqqcolon H \text{ where } \\
        (D_i)_{m, n} = 
        \begin{cases}
            1, &\text{with probability } 1/2 \text{ if } m = n \\
            0, &\text{otherwise}
        \end{cases},
        \{W_i\}_{i = j + 1}^{\ell + 1} \overset{d}{=} \{\wt{W}_i\}_{i = j + 1}^{\ell + 1}
    \end{gather*}
    Therefore, $H$ is a spherically-symmetric random vector and we condition on the following high probability bound on its length:
    \begin{align}
        \label{eq:h_l_bnd_gd_proof}
        \norm{H} &\leq \norm*{\wt{W}_{\ell + 1} \prod_{i = \ell}^{j + 2} D_i \wt{W}_i} \cdot \lprp{1 + 2\sqrt{\frac{\log \ell / \delta^\prime}{d_j}}} \leq \lprp{2}^{\ell + 1 - j}
    \end{align}
    with probability at least $1 - \delta^\prime$ and setting $\delta^\prime =
    \delta / (64\ell^2)$. Observe that the distribution of $H$ remains spherically
    symmetric even after conditioning on the above event. As in the proof of
    \cref{lem:nn_part_spec_bnd}, let $\mc{G}$ be an $\eps$-net of $f_{i_k}(\mb{B}
    (x, R))$ with $\eps$ as in the proof of \cref{lem:nn_part_spec_bnd}.
    \begin{claim}
        \label{clm:grid_approx_act_patt_app}
        We have for all $m \in \{i_k+1,\ldots, i_{k - 1}\}, y \in f_{i_k} (\mb{B} (x, R))$:
        \begin{equation*}
            \P_{\lbrb{D_m (\wt{f}_{m, i_k} (y)), D_m (\wt{f}_{m, i_k}(\wt{y}))}} \lbrb{\Tr \abs{D_m (\wt{f}_{m, i_k} (y)) - D_m (\wt{f}_{m, i_k} (\wt{y}))} \leq 4d_{i_k}} = 1
        \end{equation*}
        where $\wt{y} = \argmin_{z \in \mc{G}} \norm{z - y}$ with probability at least $1 - \delta^\prime / \ell^2$.
    \end{claim}
    \begin{proof}
        The proof is identical to the proof of \cref{clm:grid_approx_act_patt_spec_bnd_app}.
    \end{proof}

    We now break into two cases depending on how $d_{j}$ compares to $d_{i_k}$ and handle them separately.
    
    \paragraph{Case 1:} $d_j \leq d_{i_k} (\ell \log \dmax)^{20\ell}$. In this case, define the sets $\mc{S}, \mc{Q}$ as follows:
    \begin{gather*}
        \mc{S} \coloneqq \lbrb{(S_m)_{m = i_k + 1}^{j - 1}: S_m \subseteq [d_m],\ \abs{S_m} \leq 4d_{i_k}}, \\ 
        \mc{Q} \coloneqq \lbrb{Q: Q \subset [d_j],\ \abs{Q} \leq \frac{d_j}{(\ell \log \dmax)^{60\ell}}}.
    \end{gather*}
    Now, for $Q^1, Q^2 \in \mc{Q}, y \in \mc{G}, S^1, S^2 \in \mc{S}$, define the random vector:
    \begin{equation*}
        V^{j, k}_{y, Q^1, Q^2, S^1, S^2} = H (D_{Q^1} - D_{Q^2}) W_j \prod_{m = j - 1}^{i_k + 1} (D_m (\wt{f}_{m, i_k} (y)) + (D_{S^1_m} - D_{S^2_m})) W_m.
    \end{equation*}
    Note that we have from \cref{lem:dist_equiv}:
    \begin{gather*}
        V^{j, k}_{y, Q^1, Q^2, S^1, S^2} \overset{d}{=} H (D_{Q^1} - D_{Q^2}) W_j \prod_{m = j - 1}^{i_k + 1} (D_m + (D_{S^1_m} - D_{S^2_m})) W_m \text{ where } \\
        (D_m)_{i,j} = 
        \begin{cases}
            1, &\text{w.p } 1/2 \text{ if } i = j\\
            0, &\text{otherwise}
        \end{cases}.
    \end{gather*}
    Note that $\abs{\mc{Q}}^2 \cdot \abs{\mc{G}} \cdot \abs{S}^2 \leq
    (\dmax)^{64 \ell d_{i_k}}$, because of the condition on $d_j$ in
    this case. We now get:
    \begin{align*}
        &\norm*{V^{j, k}_{y, Q^1, Q^2, S_1, S_2}} \\
        &\leq \norm*{H (D_{Q^1} - D_{Q^2}) W_j \prod_{m = j - 1}^{i_k + 2} (D_m + (D_{S^1_m} - D_{S^2_m})) W_m} \cdot \lprp{1 + \sqrt{\frac{\log 1 / \delta^\dagger}{d_{i_k}}}} \\
        &\leq \norm{H (D_{Q^1} - D_{Q^2})} \cdot \prod_{m = i_k + 1}^{j} \lprp{1 + \sqrt{\frac{\log 1 / \delta^\dagger}{d_{m - 1}}}} \\
        &\leq \norm{H (D_{Q^1} - D_{Q^2})} \cdot (C \ell \log \dmax)^{(j - i_k) / 2}
    \end{align*}
    where the final inequality follows from the fact that $d_m \geq d_{i_k}$ and by setting $\delta^\dagger = \delta / (32 \cdot \ell^4 \cdot \abs{\mc{Q}}^2 \cdot \abs{\mc{G}} \cdot \abs{\mc{S}}^2)$. A union bound now implies that the previous conclusion holds for all $Q^1, Q^2 \in \mc{Q}, y \in \mc{G}, S^1, S^2 \in \mc{S}$. For the first term, we use the trivial bound:
    \begin{align*}
        \norm{H (D_{Q^1} - D_{Q^2})} &\leq \sqrt{\abs{Q^1} + \abs{Q^2}} \cdot \norm{H}_\infty \\
        &\leq 10 \cdot \sqrt{\frac{d_j}{(\ell \log \dmax)^{60\ell}}} \cdot  \sqrt{\frac{\norm{H}^2}{d_j} \cdot (\log \dmax + \log \ell / \delta)}.
    \end{align*}
    with probability at least $1 - \delta / (32 \cdot \ell^4)$ from \cref{lem:gau_unif_comp} with $M$ set to the standard basis vectors. 
    
    To conclude the proof, we get from \cref{lem:x_i_anticonc,lem:scale_pres_ml}, for all $y \text{ s.t } \norm{y - x} \leq R$:
    \begin{align*}
        &\Tr \abs{D_j (\wt{f}_j (x)) - D_j (\wt{f}_j (y))} \\
        &\leq \#\lbrb{i: \abs{\inp{(W_j)_i}{f_{j - 1} (x)}} \leq \frac{\norm{f_{j - 1} (x)}}{4 \sqrt{d_{j - 1}} \cdot (\ell \log \dmax)^{75 \ell}}} \\
        &\qquad + \frac{\norm{\wt{f}_j (x) - \wt{f}_j (y)}^2}{\norm{f_{j - 1} (x)}^2}\cdot 16 \cdot d_{j - 1} \cdot \lprp{\ell \log \dmax}^{150\ell}\\
        &\leq \frac{d_j}{2 (\ell \log \dmax)^{75 \ell}} + d_j \lprp{\frac{R^2}{\dmin}} \frac{(C \ell \log \dmax)^{\ell}}{d_{j - 1}} \cdot 16  d_{j - 1}  \lprp{\ell \log \dmax}^{150\ell} \leq \frac{d_j}{(\ell \log \dmax)^{60\ell}}.
    \end{align*}
    where the first inequality follows from the fact that for all $t > 0$:
    \begin{gather*}
        \abs{T} \leq \frac{\norm{\wt{f}_j (x) - \wt{f}_j (y)}^2}{t^2} \text{ where } \\
        T \coloneqq \lbrb{i: \abs{\inp{(W_j)_i}{f_{j - 1} (x)}} \geq t \text{ and } \sign (\inp{(W_j)_i}{f_{j - 1} (x)}) \neq \sign{(\inp{(W_j)_i}{f_{j - 1} (y)})}}.
    \end{gather*}
    Hence, we get with probability at least $1 - \delta / (16 \ell^4)$:
    \begin{gather*}
        \forall \norm{y - x} \leq R : \Tr \abs{D_j (\wt{f}_j (y)) - D_j (\wt{f}_j (x))} \leq \frac{d_j}{(\ell \log \dmax)^{60\ell}}, \\
        \forall \norm{y - x} \leq R, \forall m \in \{i_k + 1, \dots, j
        - 1\}: \Tr \abs{D_m (\wt{f}_{m} (y)) - D_m (\wt{f}_{m, i_k}
        (\wt{y}))} \leq 4d_{i_k},
    \end{gather*}
    where $\wt{y} = \argmin_{z \in \mc{G}} \norm{z - f_{i_k} (y)}$. 
    
    To conclude \eqref{eq:gd_part_bnd}, note that for all $y \in \mc{G}$ we have that $(\wt{f}_{m, i_k} (y))_n \neq 0$ almost surely on \cref{lem:scale_pres_ml}. Hence, all the $D_m (\wt{f}_{m, i_k} (y))$ are deterministic on \cref{lem:scale_pres_ml} and similarly for $x$. This immediately yields \eqref{eq:gd_part_bnd} for $y \in \mc{G}$. For $y \notin \mc{G}$, the conclusion follows from the previous discussion and \cref{clm:grid_approx_act_patt_app}. 

    \paragraph{Case 2:} $d_j \geq d_{i_k} (\ell \log \dmax)^{20\ell}$. As in the previous case, we start by defining the sets $\mc{S}$:
    \begin{equation*}
        \mc{S} = \lbrb{(S_m)_{m = i_k + 1}^{j - 1}: S_m \subseteq [d_m],\ \abs{S_m} \leq 4\cdot d_{i_k}}.
    \end{equation*}
    Now, for $y \in \mc{G}, S^1, S^2 \in \mc{S}$, consider the random matrix $M^{j,k}_{y, S^1, S^2}$ and an application of \cref{lem:dist_equiv}:
    \begin{gather*}
        M^{j, k}_{y, S^1, S^2} \coloneqq W_{j - 1} \prod_{m = j - 2}^{i_k + 1} (D_m (\wt{f}_m (y)) + (D_{S_m^1} - D_{S_m^2})) W_m \\
        M^{j, k}_{y, S^1, S^2} \overset{d}{=} \wt{W}_{j - 1} \prod_{m = j - 2}^{i_k + 1} (D_m + (D_{S_m^1} - D_{S_m^2})) \wt{W}_m.
    \end{gather*}
    We will bound the spectral norm of $M^{j,k}_{y, S^1, S^2}$ for all $y,S^1,S^2$ with high probability as follows. Let $\mc{V}$ be a $1/9$-grid of $\mb{S}^{d_{i_k}}$ and $v \in \mc{V}$. We have:
    \begin{align}
        \norm*{M^{j, k}_{y, S^1, S^2}v} &\leq \sqrt{\frac{d_{j - 1}}{d_{j - 2}}} \cdot \lprp{1 + \sqrt{\frac{\log 1 / \delta^\dagger}{d_{j - 1}}}} \cdot \norm*{\prod_{m = j - 2}^{i_k + 1} (D_m + (D_{S_m^1} - D_{S_m^2})) \wt{W}_m v} \notag \\
        &\leq 2^{(j - i_k + 1)} \sqrt{\frac{d_{j - 1}}{d_{i_k}}} \cdot \prod_{m = i_k + 1}^{j - 1} \lprp{1 + \sqrt{\frac{\log 1 / \delta^\dagger}{d_m}}} \notag \\
        &\leq \sqrt{\frac{d_{j - 1}}{d_{i_k}}} \cdot \lprp{C \ell \log \dmax}^{(j - i_k + 1) / 2} \label{eq:jk_spec_bnd}
    \end{align}
    where the final inequality follows by the fact that $d_m \geq
    d_{i_k}$ and by setting $\delta^\dagger = \delta / (32 \cdot \ell
    \cdot \abs{\mc{G}} \cdot \abs{\mc{S}}^2 \cdot \mc{V})$. By a union
    bound, the bound on the spectral norms of $M^{j,k}_{y, S^1, S^2}$
    follows. We now condition on this event and the conclusion of
    \cref{lem:scale_pres_ml} up to layer $j - 1$ for the rest of the
    proof, which as before implies that $M^{j,k}_{y, S^1, S^2}$ and $D_{j - 1} (\wt{f}_{j - 1, i_k} (y))$ are no longer random. 

    Let $\wt{M}^{j,k}_{y, S^1, S^2} = D_{j - 1}(\wt{f}_{j - 1, i_k} (y)) \cdot M^{j,k}_{y, S^1, S^2}$ and for $y \in \mc{G}$, let $\wt{y} = f_{j - 1, i_k} (y)$ and $\wt{x} = f_{j - 1} (x)$. We have by an application of \cref{lem:large_width_proof_gd_diff_bnd} and a union bound over all $y \in \mc{G}, S^1, S^2 \in \mc{S}$:
    \begin{equation*}
        \forall y \in \mc{G}, S^1, S^2 \in \mc{S}: \norm*{H^\top
        \left(D_{j} (\wt{f}_{j, i_k} (y)) - D_j (\wt{f}_{j} (x))\right)
        W_j \wt{M}^{j,k}_{y, S^1, S^2}} \leq \lprp{\frac{C}{\ell \log \dmax}}^{3\ell}
    \end{equation*}
    with probability at least $1 - \delta'/\ell^2$ by recalling that $d_j \geq d_{i_k} (\ell \log \dmax)^{20\ell}$ in this case.
    
    We now additionally condition on $\norm{H}_\infty$. We have as a consequence of \cref{lem:gau_unif_comp} that:
    \begin{equation*}
        \norm{H}_\infty \leq 4 \cdot 2^{\ell + 1 - j} \cdot \sqrt{\frac{\log d_j + \log 1 / \delta^\dagger}{d_j}}
    \end{equation*}
    with probability at least $1 - \delta^\dagger$. We condition on this event and proceed as follows. Let $T \subset [d_j]$ such that $\abs{T} \leq 4 \cdot d_{i_k}$ and $y \in \mc{G}, S^1, S^2, \in \mc{S}$ and we observe:
    \begin{equation*}
        \norm*{H^\top D_T \wt{W}_j \wt{M}^{j, k}_{y, S^1, S^2}} \leq \norm*{H^\top D_T \wt{W}_j U^{j,k}_{y, S^1, S^2}} \cdot \norm*{\wt{M}^{j, k}_{y, S^1, S^2}}
    \end{equation*}
    where $U^{j,k}_{y, S^1, S^2}$ are the left singular vectors of $\wt{M}^{j,k}_{y, S^1, S^2}$ and observe:
    \begin{align*}
        \P \lbrb{\norm*{H^\top D_T \wt{W}_j U^{j,k}_{y, S^1, S^2}} \geq t} &= \P_{\wt{Z} \thicksim \mc{N} \lprp{0, \norm{H D_T}^2 \cdot \frac{I}{d_j}}} \lbrb{\norm{\wt{Z}} \geq t} \\
        &\leq \P_{Z \thicksim \mc{N} \lprp{0, \norm{H}_\infty^2 \cdot \abs{T} \cdot \frac{I}{d_j}}} \lbrb{\norm{Z} \geq t}
    \end{align*}
    to get that (conditioned on $H$ and noting that $\mrm{rank}
    (U^{j,k}_{y, S^1, S^2}) \leq d_{i_k}$), with probability
    $1-\delta^\ddagger$:
    \begin{equation*}
        \norm*{H^\top D_T U^{j,k}_{y, S^1, S^2}} \leq \norm{H}_\infty \cdot \sqrt{\abs{T}} \cdot \sqrt{\frac{d_{i_k}}{d_j}} \cdot \lprp{1 + \sqrt{\frac{\log 1 / \delta^\ddagger}{d_{i_k}}}}
    \end{equation*}
    By setting $\delta^\ddagger = \delta / (128 \ell^2 d_j^{4d_{i_k}}
    \abs{\mc{S}}^2 \abs{\mc{G}})$ and our bounds on $d_j$,
    $\|H\|_\infty$ yield:
    \begin{equation*}
        \forall T \subset [d_j] \text{ s.t } \abs{T} \leq 4d_{i_k}, y \in \mc{G}, S^1, S^2 \in \mc{S}: \norm*{H^\top D_T U^{j,k}_{y, S^1, S^2}} \leq \frac{1}{(\ell \log \dmax)^{4\ell}}
    \end{equation*}
    with probability at least $1 - \delta / (64 \ell^2)$ again by recalling $d_j \geq d_{i_k} (\ell \log \dmax)^{20\ell}$. 

    To conclude our proof of \eqref{eq:gd_part_bnd}, we proceed similarly to the previous case. As before, for all $y \in \mc{G}$ we have that $(\wt{f}_{m, i_k} (y))_n \neq 0$ almost surely on \cref{lem:scale_pres_ml} and similarly for $x$. Hence, all the $D_m (\wt{f}_{m, i_k} (y))$ are deterministic on \cref{lem:scale_pres_ml}. This immediately yields \eqref{eq:gd_part_bnd} for $y \in \mc{G}$. For $y \notin \mc{G}$, the conclusion follows from the previous discussion and \cref{clm:grid_approx_act_patt_app}. 

    A union bound over all $j \in [\ell]$, an application of the triangle inequality with \ref{eq:grad_decomp} and \eqref{eq:gd_final_decomp} with \cref{lem:nn_part_spec_bnd} conclude the proof of the lemma.
\qed

\begin{lemma}
    \label{lem:large_width_proof_gd_diff_bnd}
    Let $L, R, r > 0, k,m,n \in \mb{N}, M \in \mb{R}^{m \times n}$
    such that $m \geq n$ and $R > r$. Furthermore, suppose $W \in
    \mb{R}^{k \times m}$ and $H$ are distributed as follows:
    \begin{equation*}  
        W = 
        \begin{bmatrix}
            w_1^\top \\
            w_2^\top \\
            \vdots \\
            w_k^\top
        \end{bmatrix}
        \text{ with }
        w_i \overset{i.i.d}{\thicksim} \mc{N} (0, I / m), 
        \text{ and } H \thicksim \mathrm{Unif} (L\cdot \mb{S}^{k - 1}).
    \end{equation*}
    Furthermore, suppose $x \in \mb{R}^m$ satisfies $\norm{x} \geq R$ and $y$ be such that $\norm{y - x} \leq r$. Then:
    \begin{gather*}
        \mb{P} \lbrb{\norm{H^\top D W M} \geq 1024 L \norm{M}
        \lprp{\sqrt{\frac{3r}{R} \log R/r \lprp{\frac{n + \log 1 / \delta^\dagger}{m}}} + \lprp{\frac{n + \log 1 / \delta^\dagger}{\sqrt{km}}}}} \leq \delta^\dagger \\
        \text{ where } D = D_x - D_y \text{ with }
        (D_x)_{i, j} = 
        \begin{cases}
            1, &\text{if } i = j \text{ and } w_i^\top x > 0 \\
            1, &\text{w.p } 1/2 \text{ if } i = j \text{ and } w_i^\top x = 0 \\
            0, &\text{otherwise}
        \end{cases}.
    \end{gather*}
\end{lemma}
\begin{proof}
    We may discard the cases where $w_i^\top x, w_i^\top y = 0$ as these form a measure $0$ set. Now, it suffices to analyze the random variable:
    \begin{equation*}
        \wt{H}^\top D W M \text{ for } \wt{H} \thicksim \mc{N} (0, 2L^2 \cdot I / k).
    \end{equation*}
    Let $V$ denote the two-dimensional subspace of $\mb{R}^m$ containing $x,y$. We now decompose the norm as follows:
    \begin{equation}
        \label{eq:diff_decomp}
        \norm{\wt{H}^\top D W M} \leq \norm{\wt{H}^\top D W \proj_V M} + \norm{\wt{H}^\top D W \proj_V^\perp M} \leq \norm{\wt{H}^\top D W \proj_V}\cdot \norm{M} + \norm{\wt{H}^\top D W \proj_V^\perp M}.
    \end{equation}
    To apply Bernstein's inequality, we first expand on the first term:
    \begin{align*}
        \left(\wt{H}^\top D W \proj_V\right)^\top
        = \sum_{i = 1}^k \wt{H}_i D_{i,i} \proj_V w_i.
    \end{align*}
    Letting $Z_i = \wt{H}_i D_{i,i} \proj_V w_i$, we note that $\E
    [Z_i] = 0$ and bound its even moments as follows, in an
    orthonormal basis $\{v_1,v_2\}$ for $V$:
    \begin{align*}
        \E \lsrs{\inp{v_j}{Z_i}^2} &= \E \lsrs{\wt{H}_i^2 \cdot D_{i,i}^2 \cdot \inp{v_j}{w_i}^2} \\
        &\leq \frac{2L^2}{k} \cdot \E \lsrs{\norm{\proj_V w_i}^2} \cdot \P \lbrb{D_{i,i} \neq 0} = \frac{2L^2}{k} \cdot \frac{2}{m} \cdot \P \lbrb{D_{i,i} \neq 0} \\
        \E \lsrs{\inp{v_j}{Z_i}^\ell} &= \E \lsrs{\wt{H}_i^\ell \cdot D_{i,i}^\ell \cdot \inp{v_j}{w_i}^\ell} \leq (\ell - 1)!! \lprp{\frac{2L^2}{k}}^{\ell / 2} \cdot \E \lsrs{\norm{\proj_V w_i}^\ell} \cdot \P \lbrb{D_{i,i} \neq 0} \\
        &\leq (\ell - 1)!! \lprp{\frac{2L^2}{k}}^{\ell / 2} \cdot \lprp{2^{\ell / 2} \cdot \E \lsrs{\inp{w_i}{v_1}^\ell + \inp{w_i}{v_2}^\ell}} \cdot \P \lbrb{D_{i,i} \neq 0} \\
        &= 2\cdot (\ell - 1)!! \lprp{\frac{4L^2}{k}}^{\ell / 2} \cdot \E \lsrs{\inp{w_i}{v_1}^\ell} \cdot \P \lbrb{D_{i,i} \neq 0} \\
        &= 2 \cdot ((\ell - 1)!!)^{2} \lprp{\frac{4L^2}{km}}^{\ell / 2} \cdot \P \lbrb{D_{i,i} \neq 0} \leq 2 \cdot \ell! \cdot \lprp{\frac{8L^2}{km}}^{\ell / 2} \cdot \P \lbrb{D_{i,i} \neq 0}.
    \end{align*}
    For odd $\ell \geq 3$ we have by similar manipulations:
    \begin{align*}
        \E \lsrs{\abs{\inp{v_j}{Z_i}}^\ell} &= \E \lsrs{\abs*{\wt{H}_i^\ell \cdot D_{i,i}^\ell \cdot \inp{v_j}{w_i}^\ell}} \leq \P \lbrb{D_{i,i} \neq 0} \cdot \E \lsrs{\abs{\wt{H}_i^\ell \cdot \norm{\proj_V w_i}^\ell}} \\
        &\leq \P \lbrb{D_{i,i} \neq 0} \cdot \sqrt{\E \lsrs{\lprp{\wt{H}_i \cdot \norm{\proj_V w_i}}^{2\ell}}} \leq \P \lbrb{D_{i,i} \neq 0} \cdot \sqrt{2 \cdot (2\ell)! \cdot \lprp{\frac{8L^2}{km}}^{\ell}} \\
        &\leq \P \lbrb{D_{i,i} \neq 0} \cdot \sqrt{2 \cdot (2\ell)!} \cdot \lprp{\frac{8L^2}{km}}^{\ell / 2} \leq \P \lbrb{D_{i,i} \neq 0} \cdot 2^\ell \ell! \cdot \lprp{\frac{8L^2}{km}}^{\ell / 2} \\
        &\leq \P \lbrb{D_{i,i} \neq 0} \cdot \ell! \cdot \lprp{\frac{32L^2}{km}}^{\ell / 2}
    \end{align*}
    A union bound and an application of \cref{thm:bernstein} (with
    $\nu=4L^2 \P \lbrb{D_{i,i} \neq 0}/m$ and $c=6L/\sqrt{km}$) gives:
    \begin{equation*}
        \norm*{\sum_{i = 1}^k Z_i} \leq 2\max_{j \in \{1,2\}} \abs*{\sum_{i = 1}^k \inp{v_j}{Z_i}} \leq 256 \cdot L \cdot \lprp{\sqrt{\frac{1}{m} \cdot \P \lbrb{D_{i,i} \neq 0} \cdot \log 1 / \delta'} + \frac{\log 1 / \delta'}{\sqrt{km}}}
    \end{equation*}
    with probability at least $1 - \delta'$.

    For the other term in \cref{eq:diff_decomp}, we use $U$ to denote the left singular subspace of $M$ and $Y$ to denote the span of the left singular vectors of $\proj_V^\perp \proj_U$ with non-zero singular values. Noting $\norm{\proj_V^\perp \proj_U} \leq 1$, we now have:
    \begin{equation*}
        \norm{\wt{H}^\top D W \proj_V^\perp M} \leq \norm{M} \cdot \norm{\wt{H}^\top D W \proj_V^\perp \proj_U} \leq \norm{M} \cdot \norm{\wt{H}^\top DW \proj_Y}.
    \end{equation*}
    We expand the term on the right as follows:
    \begin{equation*}
        \left(\wt{H}^\top D W \proj_Y\right)^\top
        = \sum_{i = 1}^k \wt{H}_i D_{i,i} \proj_Y w_i = \sum_{i = 1}^k
        D_{i,i} \cdot \lprp{\wt{H}_i \proj_Y w_i} \eqqcolon
        \sum_{i=1}^k Y_i.
    \end{equation*}
    Note that $Y$ is orthogonal to $V$ as $\proj_V (\proj_V^\perp
    \proj_U u) = 0$ for all $u$ and hence, $\wt{H}_i \proj_Y w_i$ and
    $D_{i,i}$ are independent random variables (due to $\proj_V w_i$
    and $\proj_V^\perp w_i$ being independent). Now, fix $y \in Y \text{ s.t }\norm{y} = 1$ and we bound the directional even moments of $Y_i$ with the aim of applying Bernstein's inequality as before:
    \begin{align*}
        \E \lsrs{\inp{y}{Y_i}^2} &= \E \lsrs{D_{i,i}^2 \cdot \wt{H}_i^2 \cdot \inp{y}{w_i}^2} = \frac{2L^2}{k} \cdot \frac{1}{m} \cdot \P \lbrb{D_{i,i} \neq 0} \\
        \E \lsrs{\inp{y}{Y_i}^\ell} &= \E \lsrs{D_{i,i}^\ell \cdot \wt{H}_i^\ell \cdot \inp{y}{w_i}^\ell} \\
        &= \lprp{(\ell - 1)!!}^2 \cdot \lprp{\frac{2L^2}{km}}^{\ell / 2} \cdot \P \lbrb {D_{i,i} \neq 0} \leq \ell! \cdot \lprp{\frac{2L^2}{km}}^{\ell / 2} \cdot \P \lbrb{D_{i,i} \neq 0}.
    \end{align*}
    Similarly, for odd $\ell \geq 3$, we have by similar manipulations:
    \begin{align*}
        \E \lsrs{\abs{\inp{y}{Y_i}}^\ell} &= \E \lsrs{\abs*{\wt{H}_i^\ell \cdot D_{i,i}^\ell \cdot \inp{y}{w_i}^\ell}} \leq \P \lbrb{D_{i,i} \neq 0} \cdot \E \lsrs{\abs{\wt{H}_i^\ell \cdot \inp{y}{w_i}^\ell}} \\
        &\leq \P \lbrb{D_{i,i} \neq 0} \cdot \sqrt{\E \lsrs{\abs{\wt{H}_i^\ell \cdot \inp{y}{w_i}^{2\ell}}}} \leq \P \lbrb{D_{i,i} \neq 0} \cdot \sqrt{(2\ell)! \cdot \lprp{\frac{2L^2}{km}}^{\ell}} \\
        &\leq \P \lbrb{D_{i,i} \neq 0} \cdot \sqrt{(2\ell)!} \cdot \lprp{\frac{2L^2}{km}}^{\ell / 2} \leq \P \lbrb{D_{i,i} \neq 0} \cdot 2^\ell \ell! \cdot \lprp{\frac{2L^2}{km}}^{\ell / 2} \\
        &\leq \P \lbrb{D_{i,i} \neq 0} \cdot \ell! \cdot \lprp{\frac{8L^2}{km}}^{\ell / 2}
    \end{align*}

    Now, consider a $1/3$-net of $\mc{G} \coloneqq Y \cap \S^{m - 1}$.
    We get for any $z \in Y$:
    \begin{equation*}
        \norm{z} = \max_{y \in \S^{m - 1}} \inp{y}{z}
        = \max_{y \in \S^{m - 1}} \left(\inp{y - \wt{y}}{z} + \inp{\wt{y}}{z}\right)
        \leq \frac{\norm{z}}{3} + \max_{y \in \mc{G}} \inp{y}{z}
        \implies \norm{z} \leq \frac{3}{2} \max_{y \in \mc{G}} \inp{y}{z},
    \end{equation*}
    where $\wt{y} = \argmin_{z \in \mc{G}} \norm{y - z}$.
    Since the rank of $Y$ is at most $n$, We may assume $\abs{\mc{G}} \leq (30)^{n}$\cite[Corollary 4.2.13]{vershynin}. By a union bound
    and \cref{thm:bernstein}:
    \begin{equation*}
        \norm*{\sum_{i = 1}^n Y_i} \leq 2 \max_{y \in \mc{G}} \inp*{y}{\sum_{i = 1}^n Y_i} \leq 256 L \lprp{\sqrt{\P \lbrb{D_{i,i} \neq 0} \cdot \lprp{\frac{n + \log 1 / \delta'}{m}}} + \lprp{\frac{n + \log 1 / \delta'}{\sqrt{km}}}}
    \end{equation*}
    with probability at least $1 - \delta'$. The lemma follows from the previous discussion and \cref{lem:mis_sign_prob} by picking $\delta' = \delta^\dagger / 4$ and applying a union bound.
\end{proof}

%% file: content/lower_bound_app.tex
\section{Proof of \cref{thm:lower-in-body}}
\label{sec:proof_lower}

We will now construct an architecture that when randomly initialized
approximately maps every input to a random constant times its euclidean norm. We will adopt the notation from previous sections; the output of our neural network denoted by $f$ with $\ell$ hidden layers all of fixed width $k$ is defined as follows:

\begin{gather*}
  f(x) = W_{\ell + 1} \cdot \sigma (W_\ell \cdot \sigma (\cdots \sigma
  (W_1 \cdot x))) \text{ where } \sigma (x)_i = \max \lbrb{x_i, 0}\\
  f_i (x) = \sigma (W_i \cdot \sigma (\cdots \sigma (W_1 \cdot x))) \\
  \forall i \in [\ell] \setminus \{1\}: W_{i} \in \R^{k \times k} \text{ with } (W_i)_{m,n} \overset{iid}{\thicksim} \mc{N} \lprp{0, \frac{2}{k}} \\
  W_1 \in \R^{k \times d} \text{ with } (W_1)_{m,n}
  \overset{iid}{\thicksim} \mc{N} \lprp{0, \frac{2}{d}} \qquad
  W_{\ell + 1} \thicksim \mc{N} \lprp{0, 2I/k} \tag{Lower-Bound-Init} \label{eq:lb_init}
\end{gather*}

The scaling for the intermediate layers is chosen such that it preserves the length of the input from the previous layer. We now prove the main result of the section which implies \cref{thm:lower-in-body} by a simple rescaling:

\begin{theorem}\label{thm:lower}
  Fix a sufficiently large $d\in\N$, an $\ell \geq d^3$ and
  $(\ell d)^{20} \leq k \leq \exp (\sqrt{\ell})$, and consider
  the randomly initialized neural network~\eqref{eq:lb_init}.
  There is a universal constant $C$ such
  that with probability at least $0.9$,
  \begin{gather*}
    \forall x \in \S^{d - 1} : \abs{f(x)} \geq 0.04\\
    \forall x,y \in \S^{d - 1}: \abs{f(x) - f(y)} \leq C \sqrt{\frac{\log d}{d}} 
  \end{gather*}
\end{theorem}

\begin{proof}
  We start by picking an $\eps$-net of $\S^{d - 1}$, $\mc{G}$ with
  $\eps = \lprp{\frac{1}{10\cdot 2^\ell}}^{10}$. Note we may assume
  that $\abs{\mc{G}} \leq \lprp{\frac{10}{\eps}}^d$. Now, for fixed $x
  \in \mc{G}$ and defining $f_0(x):=x$, we have:
  \begin{align*}
    \forall i \in [\ell]: &\abs{\norm{\sigma(W_i f_{i - 1} (x))} - \norm{\sigma(W_i' f_{i - 1} (x))}} \leq \norm{\sigma(W_i f_{i - 1} (x)) - \sigma(W_i' f_{i - 1} (x))}\\
    &\quad \leq \norm{W_i - W_i'} \cdot \norm{f_{i - 1} (x)} \leq \norm{W_i - W_i'}_F \cdot \norm{f_{i - 1}(x)}.
  \end{align*}
  Hence, $\norm{\sigma (W_i f_{i - 1} (x))}$ is a $\norm{f_{i - 1}
  (x)}$-Lipschitz function of $W_i$ and we get by an application of
  \cref{thm:tsirelson} (note that $(W_i)_{l,m} \thicksim \mc{N} (0, 2 / k)$) and a union bound over the $\ell$ layers:
  \begin{equation*}
    \forall i \in [\ell]: \abs*{\norm{f_i (x)} - \E [\norm{f_i (x)} \mid f_{i - 1} (x)]} \leq 8 \norm{f_{i - 1} (x)} \cdot \sqrt{\frac{\log \delta + \log \ell}{k}} 
  \end{equation*}
  with probability at least $1 - \delta$. By setting $\delta = \frac{1}{16 \cdot \abs{\mc{G}} \cdot d^{10}}$ and a union bound over all $x \in \mc{G}$, we have with probability at least $1 - 1 / (16d^{10})$:
  \begin{equation}
    \label{eq:f_i_conc_lower}
    \forall x \in \mc{G}, i \in [\ell] : \abs{\norm{f_i (x)} - \E [\norm{f_i (x)} \mid f_{i - 1} (x)]} \leq \norm{f_{i - 1} (x)} \cdot \frac{1}{2048 \cdot (\ell d)^3}.
  \end{equation}
  We also have by Jensen's inequality:
  \begin{gather*}
    \E \lsrs{\norm{f_i (x)} \mid f_{i - 1} (x)} \leq \sqrt{\E \lsrs{\norm{f_i(x)}^2 \mid f_{i - 1} (x)}} = \norm{f_{i - 1} (x)}
  \end{gather*}
  and by integrating the tail bound from \cref{thm:tsirelson}:
  \begin{equation*}
    \E \lsrs{\lprp{\norm{f_i(x)} - \E \lsrs{\norm{f_i (x)} \mid f_{i - 1} (x)}}^2 \mid f_{i - 1} (x)} \leq \frac{32}{k} \cdot \norm{f_{i - 1} (x)}^2.
  \end{equation*}
  Hence, we get:
  \begin{equation*}
    \lprp{1 - \frac{32}{k}} \norm{f_{i - 1} (x)} \leq \E \lsrs{\norm{f_i (x)} \mid f_{i - 1} (x)} \leq \norm{f_{i - 1} (x)}.
  \end{equation*}
  From the above display and \cref{eq:f_i_conc_lower} and noting $(1 + x) \leq e^x \leq (1 + 2x)$ for $0 \leq x \leq 1$, we get:
  \begin{equation*}
    \norm{f_i (x)} \leq \lprp{1 + \frac{1}{2048 \cdot (\ell d)^3}}^i \leq 1 + \frac{i}{1024 \ell^3 d^3}
  \end{equation*}
  and similarly for the lower bound:
  \begin{equation*}
    \norm{f_i (x)} \geq \lprp{1 - \frac{32}{k} - \frac{1}{2048 \cdot (\ell d)^3}}^i \geq 1 - \frac{i}{512\ell^3d^3}. 
  \end{equation*}
  Putting these together, we obtain:
  \begin{equation}
    \label{eq:len_pres_lb}
    1 - \frac{i}{512\ell^3d^3}\leq \norm{f_i (x)} \leq 1 + \frac{i}{1024 \ell^3 d^3}
  \end{equation}
  Similarly to the previous discussion, we have for all $x, y \in \mc{G}$:
  \begin{align*}
    &\forall i \in [\ell]: \abs{\norm{\sigma(W_i f_{i - 1} (x)) - \sigma(W_i f_{i - 1} (y))} - \norm{\sigma(W_i' f_{i - 1} (x)) - \sigma(W_i' f_{i - 1} (y))}} \\
    &\leq \norm{\sigma(W_i f_{i - 1} (x)) - \sigma(W_i f_{i - 1} (y)) - \sigma(W_i' f_{i - 1} (x)) + \sigma(W_i' f_{i - 1} (y))}\\
    &\leq \norm{W_i - W_i'} \cdot (\norm{f_{i - 1} (x)} + \norm{f_{i - 1} (y)}) \leq \norm{W_i - W_i'}_F \cdot (\norm{f_{i - 1}(x)} + \norm{f_{i - 1} (y)}).
  \end{align*}
  Therefore, another application of \cref{thm:tsirelson} and a union bound over the $\ell$ layers yields:
  \begin{align*}
    \forall i \in [\ell]: &\abs*{\norm{f_i (x) - f_i(y)} - \E [\norm{f_i (x) - f_i (y)} \mid f_{i - 1} (x), f_{i - 1} (y)]} \\
    &\leq 8 (\norm{f_{i - 1} (x)} + \norm{f_{i - 1} (y)}) \cdot \sqrt{\frac{\log \ell / \delta}{k}}.
  \end{align*}
  Setting $\delta = 1 / (16 \cdot \abs{\mc{G}}^2 \cdot d^{10})$ and a union bound over all $x,y \in \mc{G}$ yields with probability at least $1 - 1 / (16d^{10})$ for all $x,y \in \mc{G}, i \in [\ell]$:
  \begin{equation}
    \label{eq:lb_diff_conc}
    \abs*{\norm{f_i (x) - f_i(y)} - \E [\norm{f_i (x) - f_i (y)} \mid
    f_{i - 1} (x), f_{i - 1} (y)]} \leq \norm{f_{i - 1} (x) + f_{i - 1} (y)} \frac{1}{2048 \cdot (\ell d)^3}.
  \end{equation}

  We only need an upper bound on $\E \lsrs{\norm{f_i (x) - f_i (y)} \mid f_{i - 1} (x), f_{i - 1} (y)}$. Before, we do we need the following simple fact:
  \begin{fact}
    \label{fac:sin_cos}
    We have for some $c > 0$:
    \begin{equation*}
      \forall x \in [0, \pi]: \sin (x) - x \cos (x) \geq  \frac{(1 - \cos x)^{3/2}}{15}.
    \end{equation*}
  \end{fact}
  \begin{proof}
    Let $f(x) = \sin (x) - x \cos (x)$ we have:
    \begin{gather*}
      \forall x \in \lsrs{0, \frac{\pi}{2}}: f'(x) = x \sin (x) \geq \frac{2}{\pi} \cdot x^2 \\
      \forall x \in \lsrs{0, \pi}: f'(x) = x\sin (x) \geq 0.
    \end{gather*}
    Therefore, we have:
    \begin{gather*}
      \forall x \in \lsrs{0, \frac{\pi}{2}}: f(x) = \int_0^x f' (x) dx \geq \frac{2}{3\pi} \cdot x^3 \\
      \forall x \in \lsrs{\frac{\pi}{2}, \pi}: f(x) = \int_0^x f' (x) dx \geq \int_0^{\pi / 2} f' (x) dx \geq \frac{\pi^2}{12} \geq \frac{x^3}{36}.
    \end{gather*}
    By noting that $1 - \cos x \leq x^2 / 2$, we get:
    \begin{equation*}
      \forall x \in \lsrs{0, \pi}: f(x) \geq \frac{(1 - \cos (x))^{3 / 2}}{15}.
    \end{equation*}
  \end{proof}
  Now, defining
  $\theta_i=\arccos(\inp{f_i(x)}{f_i(y)}/(\norm{f_i(x)}\norm{f_i(y)}))$,
  we have:
  \begin{align*}
    &\E \lsrs{\norm{f_i (x) - f_i (y)} \mid f_{i - 1} (x), f_{i - 1} (y)} \\
    &\leq \sqrt{\E \lsrs{\norm{f_i (x) - f_i (y)}^2 \mid f_{i - 1} (x), f_{i - 1} (y)}} \\
    &= \sqrt{\norm{f_{i - 1} (x)}^2 + \norm{f_{i - 1} (y)}^2 + 2 \E \lsrs{\inp{f_i (x)}{f_i (y)} \mid f_{i - 1} (x), f_{i - 1} (y)}} \\
    &= \sqrt{\norm{f_{i - 1} (x)}^2 + \norm{f_{i - 1} (y)}^2 - 2 \norm{f_{i - 1} (x)} \norm{f_{i - 1} (y)} \cdot \lprp{\frac{\sin \theta_{i - 1}}{\pi} + \lprp{1 - \frac{\theta_{i - 1}}{\pi}}\cos \theta_{i - 1}}} \\
    &= \sqrt{\norm{f_{i - 1} (x) - f_{i - 1} (y)}^2 - 2 \norm{f_{i - 1} (x)} \norm{f_{i - 1} (y)} \cdot \lprp{\frac{\sin \theta_{i - 1} - \theta_{i - 1} \cos \theta_{i - 1}}{\pi}}} \\
    &\leq \sqrt{\norm{f_{i - 1} (x) - f_{i - 1} (y)}^2 - 2 \norm{f_{i - 1} (x)} \norm{f_{i - 1} (y)} \cdot \lprp{\frac{\sin \theta_{i - 1} - \theta_{i - 1} \cos \theta_{i - 1}}{\pi}}} \\
    &\leq \norm{f_{i - 1} (x) - f_{i - 1} (y)} \cdot \sqrt{1 -
    \frac{2\norm{f_{i - 1} (x)} \norm{f_{i - 1} (y)}}{\norm{f_{i - 1}
    (x) - f_{i - 1} (y)}^2} \cdot \frac{(1 - \cos \theta_{i -
    1})^{3/2}}{15 \pi}} \\
    &\leq \norm{f_{i - 1} (x) - f_{i - 1} (y)} \cdot \lprp{1 -
    \frac{\norm{f_{i - 1} (x)} \norm{f_{i - 1} (y)}}{\norm{f_{i - 1}
    (x) - f_{i - 1} (y)}^2} \cdot \frac{(1 - \cos \theta_{i -
    1})^{3/2}}{15 \pi}}. \tag{\theequation} \stepcounter{equation} \label{eq:lb_exp_evolve}
  \end{align*}
  On the event in \cref{eq:len_pres_lb}, defining $\wt{f}_i (x) = \frac{f_i (x)}{\norm{f_i (x)}}$ and similarly for $y$, we have:
  \begin{align*}
    &2(1 - \cos \theta_{i - 1}) \\*
    &= \norm{\wt{f}_{i - 1} (x)}^2 + \norm{\wt{f}_{i - 1} (y)}^2 - 2\inp{\wt{f}_{i - 1} (x)}{\wt{f}_{i - 1} (y)} \\
    &\geq \lprp{1 + \frac{1}{1024 \ell^2 d^3}}^{-2} \lprp{\norm{f_{i - 1} (x)}^2 + \norm{f_{i - 1} (y)}^2} - 2 \lprp{1 + \frac{\sgn \cos \theta_{i - 1}}{256 \ell^2 d^3}} \inp{f_{i - 1} (x)}{f_{i - 1} (y)}\\
    &\geq \norm{f_{i - 1} (x) - f_{i - 1} (y)}^2 - \frac{1}{256 \ell^2 d^3} \lprp{\norm{f_{i - 1} (x)} + \norm{f_{i - 1} (x)}}^2 \\
    &\geq \norm{f_{i - 1} (x) - f_{i - 1} (y)}^2 - \frac{1}{48 \ell^2 d^3}.
  \end{align*}
  Therefore, we get:
  \begin{equation*}
    (1 - \cos \theta_{i - 1}) \geq
    \begin{cases}
      \frac{1}{4}\cdot \norm{f_{i - 1} (x) - f_{i - 1} (y)}^2, &\text{if } \norm{f_{i - 1} (x) - f_{i - 1} (y)}^2 \geq \frac{1}{16 \ell^2 d^3} \\
      0, &\text{otherwise}
    \end{cases}.
  \end{equation*}
  By substituting into \cref{eq:lb_exp_evolve}, we have:
  \begin{align*}
    &\E \lsrs{\norm{f_i (x) - f_i (y)} \mid f_{i - 1} (x), f_{i - 1} (y)} \\
    &\leq \norm{f_{i - 1} (x) - f_{i - 1} (y)} \cdot \lprp{1 - \frac{\norm{f_{i - 1} (x)} \norm{f_{i - 1} (y)}}{\norm{f_{i - 1} (x) - f_{i - 1} (y)}^2} \cdot \frac{(1 - \cos \theta_{i - 1})^{3/2}}{15 \pi}} \\
    &\leq \norm{f_{i - 1} (x) - f_{i - 1} (y)} \cdot
    \begin{cases}
      \lprp{1 - \frac{\norm{f_{i - 1} (x) - f_{i - 1} (y)}}{1000}}, &\text{if } \norm{f_{i - 1} (x) - f_{i - 1} (y)}^2 \geq \frac{1}{16 \ell^2 d^3} \\
      1, &\text{otherwise}
    \end{cases}.
  \end{align*}
  By further substituting this into \cref{eq:lb_diff_conc}, we have:
  \begin{multline*}
    \forall i \in [\ell]: \norm{f_i (x) - f_{i} (y)} \leq \frac{1}{2048 \cdot (\ell d)^3} \\
    + \norm{f_{i - 1} (x) - f_{i - 1} (y)} \cdot 
    \begin{cases}
      \lprp{1 - \frac{\norm{f_{i - 1} (x) - f_{i - 1} (y)}}{1000}}, &\text{if } \norm{f_{i - 1} (x) - f_{i - 1} (y)}^2 \geq \frac{1}{16 \ell^2 d^3} \\
      1, &\text{otherwise}
    \end{cases}.
  \end{multline*}
  For the rest of the proof, we break into two cases:
  
  \noindent \textbf{Case 1:} $\norm{f_{j} (x) - f_{j} (y)}^2 \leq \frac{1}{4 \ell^2 d^3}$ for some $j \in [\ell]$. In this case, we simply show that:
  \begin{equation*}
    \forall i \geq [j]: \norm{f_i (x) - f_i (y)}^2 \leq \frac{1}{\ell^2 d^3}.
  \end{equation*}
  Suppose for the sake of contradiction, assume the contrary and let $i^*$ be the least index greater than $j$ such that the above condition was violated. We have:
  \begin{multline*}
    \norm{f_{i^*} (y) - f_{i^*} (x)} \leq \frac{1}{2048 \cdot (\ell d)^3} \\
    + \norm{f_{i^* - 1} (x) - f_{i^* - 1} (y)} \cdot 
    \begin{cases}
      \lprp{1 - \frac{\norm{f_{i^* - 1} (x) - f_{i^* - 1} (y)}}{1000}}, &\text{if } \norm{f_{i^* - 1} (x) - f_{i^* - 1} (y)}^2 \geq \frac{1}{16 \ell^2 d^3} \\
      1, &\text{otherwise}
    \end{cases}.
  \end{multline*}
  Now, if $\norm{f_{i^* - 1} (x) - f_{i^* - 1} (y)}^2 \leq \frac{1}{4 (\ell^2 d^3)}$, we have:
  \begin{equation*}
    \norm{f_{i^*} (y) - f_{i^*} (x)} \leq \frac{1}{2048 \cdot (\ell d)^3} + \frac{1}{2 \ell d^{3/2}}
  \end{equation*}
  yielding the contradiction. Alternatively, we have:
  \begin{align*}
    \norm{f_{i^*} (y) - f_{i^*} (x)} &\leq \frac{1}{2048 \cdot (\ell d)^3} + \norm{f_{i^* - 1} (x) - f_{i^* - 1} (y)} - \frac{\norm{f_{i^* - 1} (x) - f_{i^* - 1} (y)}^2}{1000} \\
    &\leq \norm{f_{i^* - 1} (x) - f_{i^* - 1} (y)}
  \end{align*}
  yielding a contradiction in this case as well.

  \noindent \textbf{Case 2:} $\norm{f_{j} (x) - f_{j} (y)}^2 \geq \frac{1}{4 \ell^2 d^3}$ for all $j \in [\ell]$. In this case, we have:
  \begin{equation*}
    \forall i \in [\ell]: \norm{f_{i} (x) - f_{i} (y)} \leq \norm{f_{i - 1} (x) - f_{i - 1} (y)} \cdot \lprp{1 - \frac{\norm{f_{i - 1} (x) - f_{i - 1} (y)}}{2000}}.
  \end{equation*}
  Here, we prove: 
  \begin{equation*}
    \norm{f_\ell (x) - f_\ell (y)} \leq \frac{1}{\sqrt{\ell} d}.
  \end{equation*}
  Suppose again for the sake of contradiction that the above condition is violated, then we have:
  \begin{equation*}
    \norm{f_\ell (x) - f_\ell (y)} \leq \norm{x - y} \cdot \lprp{1 - \frac{1}{2000 \sqrt{\ell} d}}^\ell \leq 2 \cdot \exp \lbrb{- \frac{\sqrt{\ell}}{2000 d}}
  \end{equation*}
  thus yielding a contradiction. 

  Therefore, we may assume from \cref{lem:gau_spec_conc}:
  \begin{gather*}
    \forall x, y \in \mc{G}: \norm{f_\ell (x) - f_\ell (y)} \leq \frac{1}{\sqrt{\ell} d} \\
    \forall x \in \mc{G}, i \in [\ell]: 1 - \frac{i}{512\ell^3d^3}\leq \norm{f_i (x)} \leq 1 + \frac{i}{1024 \ell^3 d^3} \\
    \forall i \in [\ell] : \norm{W_i} \leq 4.
  \end{gather*}
  Furthermore, we have with probability $1 - \delta$:
  \begin{equation*}
    \abs{f(x) - f(y)} = \abs{W_{\ell + 1} (f_\ell (x) - f_\ell (y))} \leq 2 \norm{f_\ell (x) - f_\ell (y)} \sqrt{\log 1 / \delta} \leq 2 \sqrt{\frac{\log 1 / \delta}{\ell d^2}}.
  \end{equation*}
  By picking $\delta = 1 / (64 \cdot d^{10} \cdot \abs{\mc{G}}^2)$, we get with probability at least $1 - 1 / (64d^{10})$:
  \begin{equation*}
    \abs{f(x) - f(y)} \leq C \sqrt{\frac{\log d}{d}}.
  \end{equation*}

  For $x \notin \mc{G}$, let $\wt{x} = \argmin_{y \in \mc{G}} \norm{x - y}$ and we have:
  \begin{equation*}
    \norm{f_\ell (x) - f_\ell (\wt{x})} \leq 4^\ell \eps \leq \frac{1}{2^\ell}
  \end{equation*}
  and 
  \begin{equation*}
    \norm{f (x) - f(\wt{x})} \leq \frac{1}{2^{\ell}} \cdot \norm{W_{\ell + 1}}.
  \end{equation*}
  On the event, $\norm{W_{\ell + 1}} \leq 2 \sqrt{k}$, this yields the conclusion:
  \begin{equation*}
    \forall x, y \in \S^{d - 1}: \abs{f(x) - f(y)} \leq C \sqrt{\frac{\log d}{d}}
  \end{equation*}
  with probability at least $1 - 1 / d^{10}$.

  Finally, we have by the anti-concentration of Gaussians that with probability at least $0.95$ for fixed $x \in \mc{G}$:
  \begin{equation*}
    \abs{f(x)} \geq 0.05.
  \end{equation*}

  A union bound and the above two displays concludes the proof.
\end{proof}

%% file: content/misc.tex
\section{Miscellaneous Results}
\label{sec:misc}

We restate standard results used in our analysis. We start with a fact on gaussian random matrices.

\begin{lemma}
    \label{lem:gau_spec_conc}
    Let $A$ be an $m \times n$ random matrix with $A_{i,j} \overset{i.i.d}{\thicksim} \mc{N} (0, 1)$. Then, we have that:
    \begin{equation*}
        \norm{A} \leq 3 (\sqrt{m} + \sqrt{n} + \sqrt{\log 1 / \delta})
    \end{equation*}
    with probability at least $1 - \delta$.
\end{lemma}
\begin{proof}
    We may assume without loss of generality that $m \leq n$. Let $\mc{G}$ be a $1/3$-grid of $\S^{m - 1}$ and we have:
    \begin{equation*}
        \norm{A} = \max_{u \in \S^{m - 1}} \norm{u^\top A} = \max_{u \in \S^{m - 1}} \norm{(u - \wt{u})^\top A + \wt{u}^\top A} \leq \frac{\norm{A}}{3} + \max_{u \in \mc{G}} \norm{u^\top A} 
    \end{equation*}
    where $\wt{u} = \argmin_{v \in \mc{G}} \norm{u - \wt{u}}$ which implies:
    \begin{equation*}
        \norm{A} \leq \frac{3}{2} \cdot \max_{u \in \mc{G}} \norm{u^\top A}.
    \end{equation*}
    Note, we may assume that $\abs{G} \leq (10)^m$ and $u^\top A \thicksim \mc{N} (0, I)$ and we have from \cref{thm:tsirelson}:
    \begin{equation*}
        \forall \norm{u} = 1: \P \lbrb{\norm{u^\top A} \leq \sqrt{n} + \sqrt{\log 1 / \delta}} \geq 1 - \delta'.
    \end{equation*}
    Setting $\delta' = \delta / \mc{G}$ and a union bound yields:
    \begin{equation*}
        \max_{u \in \mc{G}} \norm{u^\top A} \leq \sqrt{n} + \sqrt{m \log (10)} + \sqrt{\log 1 / \delta} \leq 2 (\sqrt{m} + \sqrt{n} + \sqrt{\log 1 / \delta})
    \end{equation*}
    with probability at least $1 - \delta$ which yields the lemma from the previous discussion.
\end{proof}

We also recall Bernstein's Inequality used frequently throughout our analysis.
\begin{theorem}{\cite[Theorem 2.1]{blm}}
    \label{thm:bernstein}
    Let $X_1, \dots, X_n$ be $n$ independent real-valued random variables. Assume there exist positive numbers $\nu$ and $c$ such that:
    \begin{equation*}
        \sum_{i = 1}^n \E \lsrs{X_i^2} \leq \nu \text{ and } \sum_{i = 1}^n \E \lsrs{\abs{X_i}^q} \leq \frac{q!}{2} \nu c^{q - 2} \text{ for all } q \geq 3.
    \end{equation*}
    Then, we have:
    \begin{equation*}
        \P \lbrb{ \sum_{i = 1}^n (X_i - \E \lsrs{X_i}) \geq \sqrt{2\nu t} + ct} \leq e^{-t}.
    \end{equation*}
\end{theorem}
We additionally recall the Tsirelson-Ibragimov-Sudakov inequality.
\begin{theorem}{\cite[Theorem 5.6]{blm}}
    \label{thm:tsirelson}
    Let $X = (X_1, \dots, X_n)$ be a vector of $n$ independent standard normal random variables. Let $g: \R^n \to \R$ denote a $L$-Lipschitz function. Then, we have:
    \begin{equation*}
        \forall t \geq 0: \P \lbrb{g(X) - \E g(X) \geq t} \leq \exp \lbrb{- \frac{t^2}{2L^2}}.
    \end{equation*}
\end{theorem}
We frequently use this inequality with the
$1$-Lipschitz functions $g(X)=\pm\|X\|$. Combining
Theorem~\ref{thm:tsirelson} with Gautschi's inequality, which
implies $\E\|X\|/\sqrt{n}\in(\sqrt{1-1/n},\sqrt{1+1/n}) \subset
(1-1/\sqrt{n}, 1+1/\sqrt{n})$, gives, for $\delta\in(0,1)$,
\begin{align*}
  \P\lbrb{\|X\|\ge\sqrt{n}+\sqrt{2\ln 1/\delta} + 1} &\le\delta, &
  \P\lbrb{\|X\|\le\sqrt{n}-\sqrt{2\ln 1/\delta} - 1} &\le\delta.
\end{align*}
This immediately implies the following corollary.
\begin{corollary}\label{cor:tsirelson}
For $\delta\in(0,1)$, a matrix $W_i\in\R^{d_i\times d_{i-1}}$
with independent $\mathcal{N}(0,1/d_{i-1})$ entries,
and vectors $u\in\R^{d_i}$ and $v\in\R^{d_{i-1}}$, both of
the following events have probability at least $1-\delta$:
  \begin{align*}
    \left|\left\|u^\top W_i\right\|-\|u\|\right|
      &\le \|u\|\frac{\sqrt{2\ln(2/\delta)}+1}{\sqrt{d_{i-1}}}, &
    \left|\left\|W_iv\right\|-\|v\|\sqrt{\frac{d_i}{d_{i-1}}}\right|
      &\le \|v\|\frac{\sqrt{2\ln(2/\delta)}+1}{\sqrt{d_{i-1}}}.
  \end{align*}
\end{corollary}

We prove another simple lemma.
\begin{lemma}
    \label{lem:mis_sign_prob}
    Let $x, y \in \R^d, R, r \in \R_+$ be such that $\norm{x} \geq R > 0$ and $\norm{x - y} \leq r$ with $R \geq r$. Then, we have:
    \begin{equation*}
        \P_{w \sim \mc{N} (0, I)} \lbrb{\sign (w^\top x) \neq \sgn (w^\top y)} \leq \frac{3r}{R} \cdot \sqrt{\log R / r}.
    \end{equation*}
\end{lemma}
\begin{proof}
    We have that $X \coloneqq w^\top x \thicksim \mc{N} (0, \norm{x}^2)$ and $Z \coloneqq w^\top (y - x) \thicksim \mc{N} (0, \norm{y - x}^2)$. Hence, we have:
    \begin{equation*}
        \P \lbrb{\sign (w^\top x) \neq \sign (w^\top y)} \leq \P \lbrb{\abs{Z} \geq \abs{X}}.
    \end{equation*}
    We have for any $\delta > 0$:
    \begin{equation*}
        \P \lbrb{\abs{Z} \geq 2 r \sqrt{\log 1 / \delta}} \leq \delta \text{ and } \P \lbrb{\abs{X} \leq 2 r \sqrt{\log 1 / \delta}} \leq \frac{4 r \sqrt{\log 1 / \delta}}{\sqrt{2\pi} R} \leq 2 \cdot \frac{r}{R} \cdot \sqrt{\log 1 / \delta}.
    \end{equation*}
    By a union bound, we have:
    \begin{equation*}
        \forall \delta > 0: \P \lbrb{\abs{Z} \geq \abs{X}} \leq \delta + \frac{2r}{R} \cdot \sqrt{\log 1 / \delta}.
    \end{equation*}
    By setting $\delta = r / R$, we get the conclusion of the lemma.
\end{proof}

\begin{lemma}
    \label{lem:gau_unif_comp}
    Let $d \in \N$ with $d \geq 40$. Then, we have for all $k \in \N, r \geq 0, t \geq 0, M \in \R^{k \times d}$:
    \begin{equation*}
        \P_{Z \thicksim \mrm{Unif} (r \S^{d - 1})} \lbrb{\norm{M Z} \geq t} \leq 2 \P_{Y \thicksim \mc{N} (0, 2r^2 I / d)} \lbrb{\norm{M Y} \geq t}.
    \end{equation*}
\end{lemma}
\begin{proof}
    The proof follows from the following manipulations:
    \begin{align*}
        \P_{Y \thicksim \mc{N} (0, 2r^2 I / d)} \lbrb{\norm{M Y} \geq t} &= \E_{Y \thicksim \mc{N} (0, 2r^2 I / d)} \lsrs{\bm{1} \lbrb{\norm{MY} \geq t}} \\
        &\geq \E_{Y \thicksim \mc{N} (0, 2r^2 I / d)} \lsrs{\bm{1} \lbrb{\norm{MY} \geq t} \bm{1} \lbrb{\norm{Y} \geq r}} \\
        &= \E_{Y \thicksim \mc{N} (0, 2r^2 I / d)} \lsrs{\bm{1} \lbrb{\frac{\norm{MY}}{\norm{Y}} \geq \frac{t}{\norm{Y}}} \bm{1} \lbrb{\norm{Y} \geq r}} \\
        &\geq \E_{Y \thicksim \mc{N} (0, 2r^2 I / d)} \lsrs{\bm{1} \lbrb{\frac{\norm{MY}}{\norm{Y}} \geq \frac{t}{r}} \bm{1} \lbrb{\norm{Y} \geq r}} \\
        &\geq \E_{Y \thicksim \mc{N} (0, 2r^2 I / d)} \lsrs{\P_{Z \thicksim \mrm{Unif} (r \S^{d - 1})} \lbrb{\norm{MZ} \geq t} \bm{1} \lbrb{\norm{Y} \geq r}} \\
        &\geq \P_{Z \thicksim \mrm{Unif} (r \S^{d - 1})} \lbrb{\norm{MZ} \geq t} \cdot \P_{Y \thicksim \mc{N} (0, 2r^2 I / d)} \lbrb{\norm{Y} \geq t} \\
        &\geq \frac{1}{2} \cdot \P_{Z \thicksim \mrm{Unif} (r \S^{d - 1})} \lbrb{\norm{MZ} \geq t}
    \end{align*}
    concluding the proof of the lemma.
\end{proof}
